\documentclass{article}

\usepackage[main, final]{neurips_2026}

\usepackage[utf8]{inputenc} %
\usepackage[T1]{fontenc}    %
\usepackage{hyperref}       %
\usepackage{url}            %
\usepackage{booktabs}       %
\usepackage{amsfonts}       %
\usepackage{nicefrac}       %
\usepackage{microtype}      %
\usepackage{xcolor}         %
\usepackage{comment}
\newcommand{\newc}{\newcommand}
\newc{\E}{\mathbb{E}}
\newc{\V}{\mbox{V}}
\newc{\N}{\mbox{N}}
\newc{\Bern}{\mbox{Bern}}
\newc{\Po}{\mbox{Po}}
\newc{\IG}{\mbox{IG}}
\newc{\Gam}{\mbox{Gam}}
\newc{\bdp}{\mathbf{p}}
\newc{\bdt}{\mathbf{t}}
\newc{\Normal}{\mathcal{N}}
\newc{\Expectation}{\mathbb{E}}
\newc{\odds}{\mbox{odds}}
\newc{\stderr}{\mbox{s.e.}}
\newc{\logit}{\mbox{logit}}
\newc{\sign}{\mbox{sign}}
\newc{\SD}{\mbox{SD}}
\newc{\qtilde}{\tilde{q}}
\newc{\deltahat}{\hat{\delta}}
\newc{\shat}{\hat{s}}
\newc{\bdq}{\mathbf{q}}
\newc{\bdmu}{\mbox{\boldmath $\mu$}}
\newc{\bdSigma}{\mbox{\boldmath $\Sigma$}}
\newc{\bdLambda}{\mbox{\boldmath $\Lambda$}}
\newc{\bdmuhat}{\mbox{\boldmath $\hat{\mu}$}}
\newc{\bdeta}{\mbox{\boldmath $\eta$}}
\newc{\bdtheta}{\mbox{\boldmath $\theta$}}
\newc{\bdbeta}{\mbox{\boldmath $\beta$}}
\newc{\bdgamma}{\mbox{\boldmath $\gamma$}}
\newc{\bdbetahat}{\mbox{\boldmath $\hat{\beta}$}}
\newc{\bdgammahat}{\mbox{\boldmath $\hat{\gamma}$}}
\newc{\bdthetahat}{\mbox{\boldmath $\hat{\theta}$}}
\newc{\bdvareps}{\mbox{\boldmath $\varepsilon$}}
\newc{\bdzero}{\mbox{\boldmath $0$}}
\newc{\bdone}{\mbox{\boldmath $1$}}
\newc{\bdnu}{\mbox{\boldmath $\nu$}}
\newc{\bdell}{\mbox{\boldmath $\ell$}}
\newc{\bdxi}{\mbox{\boldmath $\xi$}}
\newc{\bdomega}{\mbox{\boldmath $\omega$}}
\newc{\bdepsilon}{\mbox{\boldmath $\varepsilon$}}
\newc{\bdI}{\mathbf{I}}
\newc{\bdP}{\mbox{\boldmath $P$}}
\newc{\bdX}{\mbox{\boldmath $X$}}
\newc{\bdA}{\mbox{\boldmath $A$}}
\newc{\bdB}{\mbox{\boldmath $B$}}
\newc{\bdC}{\mbox{\boldmath $C$}}
\newc{\bdD}{\mbox{\boldmath $D$}}
\newc{\bdG}{\mbox{\boldmath $G$}}
\newc{\bdJ}{\mbox{\boldmath $J$}}
\newc{\bdK}{\mbox{\boldmath $K$}}
\newc{\bda}{\mbox{\boldmath $a$}}
\newc{\bdb}{\mbox{\boldmath $b$}}
\newc{\bdc}{\mbox{\boldmath $c$}}
\newc{\bde}{\mbox{\boldmath $e$}}
\newc{\bdu}{\mbox{\boldmath $u$}}
\newc{\bdv}{\mbox{\boldmath $v$}}
\newc{\bdx}{\mbox{\boldmath $x$}}
\newc{\bdy}{\mbox{\boldmath $y$}}
\newc{\bdz}{\mbox{\boldmath $z$}}
\newc{\bdr}{\mbox{\boldmath $r$}}
\newc{\bdQ}{\mbox{\boldmath $Q$}}
\newc{\bdR}{\mbox{\boldmath $R$}}
\newc{\bdY}{\mbox{\boldmath $Y$}}
\newc{\bdT}{\mbox{\boldmath $T$}}
\newc{\bdW}{\mbox{\boldmath $W$}}
\newc{\bdH}{\mbox{\boldmath $H$}}
\newc{\bdL}{\mbox{\boldmath $L$}}
\newc{\bdU}{\mbox{\boldmath $U$}}
\newc{\bdV}{\mbox{\boldmath $V$}}
\newc{\Multinom}{\mbox{Multinom}}
\newc{\tbar}{\bar{t}}
\newc{\Var}{\mbox{Var}}
\newc{\var}{\mbox{var}}
\newc{\diag}{\mbox{diag}}
\newc{\tr}{\mbox{tr}}
\newc{\phat}{\hat{p}}
\newc{\Xbar}{\bar{X}}
\newc{\xbar}{\bar{x}}
\newc{\Ybar}{\bar{Y}}
\newc{\ybar}{\bar{y}}
\newc{\dbar}{\bar{d}}
\newc{\yhat}{\hat{y}}
\newc{\bdyhat}{\mbox{\boldmath $\hat{y}$}}
\newc{\ytil}{\tilde{y}}
\newc{\ytilde}{\tilde{y}}
\newc{\ftil}{\tilde{f}}
\newc{\Ho}{\mbox{\bf H}_o}
\newc{\Ha}{\mbox{\bf H}_a}
\newc{\phatYX}{\phat_Y - \phat_X}
\newc{\SSG}{\mbox{SSG}}
\newc{\SSB}{\mbox{SSB}}
\newc{\SSE}{\mbox{SSE}}
\newc{\SST}{\mbox{SST}}
\newc{\SSR}{\mbox{SSR}}
\newc{\SSAB}{\mbox{SSAB}}
\newc{\MSG}{\mbox{MSG}}
\newc{\MSB}{\mbox{MSB}}
\newc{\MSE}{\mbox{MSE}}
\newc{\MST}{\mbox{MST}}
\newc{\MSAB}{\mbox{MSAB}}
\newc{\dfE}{\mbox{dfE}}
\newc{\dfG}{\mbox{dfG}}
\newc{\dfB}{\mbox{dfB}}
\newc{\dfT}{\mbox{dfT}}
\newc{\dfAB}{\mbox{dfAB}}
\newc{\muhat}{\hat{\mu}}
\newc{\betahat}{\hat{\beta}}
\newc{\alphahat}{\hat{\alpha}}
\newc{\etahat}{\hat{\eta}}
\newc{\phihat}{\hat{\phi}}
\newc{\sigmahat}{\hat{\sigma}}
\newc{\cl}{\centerline}
\newc{\redtitle}[1]{ {\color{red}\und{#1}:} }
\newc{\bluetitle}[1]{ {\color{blue}\und{#1}:} }
\newc{\magentatitle}[1]{ {\color{magenta}\und{#1}:} }
\newc{\R}{\mathbb{R}}
\newc{\trans}{^\mathsf{T}}
\newc{\xtx}{\bdX\trans\bdX}
\newc{\xxtxx}{\bdX(\xtx)^{-1}\bdX\trans}
\newc{\argmin}{\operatornamewithlimits{argmin}}
\newc{\argmax}{\operatornamewithlimits{argmax}}
\newc{\setS}{\mathcal{S}}
\newc{\toP}{\overset{p}{\to}}
\newc{\toD}{\overset{d}{\to}}
\newc{\eqD}{\overset{d}{=}}
\newc{\toAS}{\overset{a.s.}{\to}}
\newc{\simIID}{\overset{\text{i.i.d.}}{\sim}}
\newc{\simIND}{\overset{\text{ind.}}{\sim}}
\newc{\thetahat}{\hat{\theta}}
\newc{\bds}{\mathbf{s}}
\newc{\thetatilde}{\tilde{\theta}}
\newc{\Phat}{\hat{P}}
\newc{\Qhat}{\hat{Q}}
\newc{\tny}{\small}
\newc{\Col}{\textnormal{Col}}
\newc{\mutilde}{\tilde{\mu}}
\newc{\vecst}{\textnormal{vec}}
\newc{\indep}{\perp \!\!\! \perp}
\newc{\Cov}{\textnormal{Cov}}
\newc{\gtilde}{\tilde{g}}
\newc{\rank}{\textnormal{rank}}
\newc{\Cor}{\textnormal{Cor}}
\newc{\Lcol}{\mathcal{L}_{col}}
\newc{\Lrow}{\mathcal{L}_{row}}
\newc{\Ell}{\mathcal{L}}
\newc{\Norm}{\mathcal{N}}
\newc{\xtilde}{\tilde{x}}
\newc{\Xtilde}{\tilde{X}}
\newc{\Atilde}{\tilde{A}}
\newc{\Xhat}{\hat{X}}
\newc{\Vhat}{\hat{V}}
\newc{\Vperp}{V^\perp}
\newc{\Rhat}{\hat{R}}
\newc{\Shat}{\hat{S}}
\newc{\uhat}{\hat{u}}
\newc{\hhat}{\hat{h}}
\newc{\Yhat}{\hat{Y}}
\newc{\bpage}{\eject}
\newc{\bdg}{\mathbf{g}}
\newc{\bdZ}{\mathbf{Z}}
\newc{\Rn}{\mathbb{R}^n}
\newc{\Sigmahat}{\hat{\Sigma}}
\newc{\Ysubj}{Y_{(j)}}
\newc{\Rsubj}{R_{(j)}}
\newc{\A}{\mathcal{A}}
\newc{\B}{\mathcal{B}}
\newc{\Gammasubj}{\Gamma_{(j)}}
\newc{\Btilde}{\tilde{\mathcal{B}}}
\newc{\Xitilde}{\tilde{\Xi}}
\newc{\xitilde}{\tilde{\xi}}
\newc{\Gammatilde}{\tilde{\Gamma}}
\newc{\gammatilde}{\tilde{\gamma}}
\newc{\Lambdatilde}{\tilde{\Lambda}}
\newc{\Ls}{\mathcal{L}}
\newc{\Lambdasubj}{\Lambda_{(j)}}
\newc{\Msubj}{M_{(j)}}
\newc{\Lsubj}{L_{(j)}}
\newc{\Ytilde}{\tilde{Y}}
\newc{\Ytildehat}{\hat{\tilde{Y}}}
\newc{\xtildebar}{\tilde{\xbar}}
\newc{\F}{\mathcal{F}}
\newc{\M}{\mathcal{M}}
\newc{\G}{\mathcal{G}}
\newc{\Dtilde}{\tilde{D}}
\newc{\Utilde}{\tilde{U}}
\newc{\Vtilde}{\tilde{V}}
\newc{\betatilde}{\tilde{\beta}}
\newc{\fhat}{\hat{f}}
\newc{\Rtilde}{\tilde{R}}
\newc{\Qtilde}{\tilde{Q}}
\newc{\gammahat}{\hat{\gamma}}
\newc{\X}{\mathcal{X}}
\newc{\Y}{\mathcal{Y}}
\newc{\abar}{\bar{a}}
\newc{\phitilde}{\tilde{\phi}}
\newc{\ahat}{\hat{a}}
\newc{\bhat}{\hat{b}}
\newc{\bdf}{\mathbf{f}}
\newc{\bdh}{\mathbf{h}}
\newc{\Pois}{\text{Pois}}
\newc{\bdlambda}{\boldsymbol{\lambda}}
\newc{\Mult}{\text{Mult}}
\newc{\liminfc}{\underline{\lim}}
\newc{\limsupc}{\overline{\lim}}
\newc{\pbar}{\overline{p}}
\newc{\bdalpha}{\boldsymbol{\alpha}}
\newc{\psibar}{\overline{\psi}}
\newc{\bdxbar}{\overline{\mathbf{x}}}
\newc{\bdd}{\mathbf{d}}
\newc{\Wtilde}{\tilde{W}}
\newc{\bdZtilde}{\tilde{\mathbf{Z}}}
\newc{\ttilde}{\tilde{t}}
\newc{\Ztilde}{\tilde{Z}}
\newc{\pihat}{\hat{\pi}}
\newc{\betahatridge}{\betahat^{\lambda}_{\text{RR}}}
\newc{\epsilonhat}{\hat{\epsilon}}
\newc{\zhat}{\hat{z}}
\newc{\qhat}{\hat{q}}
\newc{\Bhat}{\hat{B}}

\newc{\hh}{\mathcal{H}}
\newc{\Unif}{\text{Unif}}
\newc{\psihat}{\hat{\psi}}
\newc{\Expo}{\text{Expo}}
\newc{\bdxtilde}{\tilde{\mathbf{x}}}
\newc{\toDPn}{\overset{P_n}{\rightsquigarrow}}
\newc{\toDQn}{\overset{Q_n}{\rightsquigarrow}}
\newc{\Fhat}{\hat{F}}
\newc{\bdO}{\mathbf{O}}
\newc{\bdn}{\mathbf{n}}
\newc{\Uhat}{\hat{U}}
\usepackage{xurl}
\usepackage{dblfloatfix}
\usepackage{float}
\usepackage{amsmath}
\usepackage{amssymb}
\usepackage{mathtools}
\usepackage{amsthm}
\usepackage{enumitem}
\usepackage{graphicx}
\theoremstyle{plain}
\newtheorem{theorem}{Theorem}[section]
\newtheorem{proposition}[theorem]{Proposition}
\newtheorem{lemma}[theorem]{Lemma}

\theoremstyle{definition}

\newtheorem{assumption}[theorem]{Assumption}
\theoremstyle{remark}

\usepackage{algorithm}
\usepackage{algorithmic}
\usepackage{placeins}

\title{Efficient Evaluation of LLM Performance with Statistical Guarantees}

\author{%
  Skyler Wu\\
  Department of Statistics\\
  Stanford University\\
  \texttt{skylerw@stanford.edu}\\
  \And
  Yash Nair\\
  Department of Statistics\\
  Stanford University\\
  \texttt{yashnair@stanford.edu}\\
  \AND
  Emmanuel J. Candès\\
  Department of Statistics\\
  Stanford University\\
  \texttt{candes@stanford.edu}\\
}

\begin{document}

\maketitle

\begin{abstract}
  Exhaustively evaluating many large language models (LLMs) on a large suite of benchmarks is expensive. We cast benchmarking as finite-population inference and, under a fixed query budget, seek tight confidence intervals (CIs) for model accuracy with asymptotic frequentist coverage. We propose \textit{Factorized Active Querying} (FAQ), which (a) leverages historical information through a Bayesian factor model; (b) adaptively selects questions using a hybrid variance-reduction/active-learning sampling policy; and (c) maintains validity through \textit{Pro-Active Inference}---a finite-population extension of active inference \citep{zrnic2024active} that enables direct question selection while preserving asymptotic frequentist coverage. In our experiments, FAQ delivers up to $5\times$ effective sample size gains over strong baselines on two benchmark suites, across varying historical-data missingness levels, with negligible computational overhead: this means that it matches the CI width of uniform sampling while using up to $5\times$ fewer queries. We release our source code and our curated datasets to support reproducibility and future research.
\end{abstract}

\section{Introduction}

\textbf{Motivation}\;\;\; Large language models (LLMs) are rapidly moving from research demos to deployed systems in domains such as healthcare, law, and education \citep{chen2024survey, li2023large, dehghani2025large, kasneci2023chatgpt, artsi2025large, achiam2023gpt}. As adoption grows, so does the number of candidate models organizations must evaluate and monitor. This scale is visible in public ecosystems: HuggingFace hosts at least $1.86$M models, including $300$K+ tagged as text-generation---many of them fine-tuned variants of certain base models \citep{laufer2025anatomy}. Indeed, private institutions may maintain hundreds to thousands of internal variants, each exhibiting their own query-dependent strengths and weaknesses. As post-training becomes cheaper, the number of viable candidate models per deployment decision will only grow.

A common practice is to evaluate candidate models' performances on finite, curated benchmark suites. Such suites increasingly include questions that (a) require expensive domain-expert adjudication \citep{szymanski2025limitations, li2025preference, enguehard2025lemaj, croxford2025current}; (b) involve multi-turn interactions (e.g., helpfulness and safety) \citep{he2024multi, kwan2024mt}; and/or (c) demand long-form reasoning with verification \citep{bai2024longbench2, ding2025longreasonarena}. With tens of thousands of questions (e.g., $12$K on MMLU-Pro \citep{wang2024mmlu}), and hundreds to thousands of candidate models, exhaustively evaluating every model on every question is expensive in API costs, runtime, and expert annotation \citep{zhou2025speeding}.\footnote{Although we use ``questions" throughout for readability, the same finite-bank formulation applies to any suites of interactive tasks or trajectories that report task-level success indicators.} In practice, organizations may instead resort to ``vibes"-based judgments on ad-hoc subsets \citep{mazzoni2025}.

Meanwhile, many organizations have access to \textit{historical performance data}---from prior benchmark runs (e.g., the Open LLM Leaderboard \citep{open-llm-leaderboard-v2}) and/or production evaluation pipelines. As we show, this data can be leveraged to make model evaluation substantially faster and cheaper.

\textbf{Problem description}\;\;\; Consider a benchmark suite of $N_q$ evaluation units---questions, tasks, or trajectories---viewed as a fixed finite population. For a new model, let $z_j \in \{ 0, 1\}$ indicate whether it succeeds on unit $j$. We wish to estimate the \textit{model's accuracy} over the entire finite evaluation bank: $\theta := \frac{1}{N_q} \sum_{j=1}^{N_q} z_{j} \in [0, 1].$

We also assume access to historical data $H \in \{ 0, 1, \text{NA} \}^{N_{\text{old}} \times N_q}$ from $N_{\text{old}}$ prior models (e.g., previously benchmarked and/or observed in deployment), potentially with many missing entries (i.e., NA) due to resource and deployment constraints. Here, $H_{ij}$ records historical model $i$'s outcome on question $j$: $1$ for correct, $0$ for incorrect, and NA for unobserved. Given a new model and budget $n_b$, the problem is: how should we adaptively select which $n_b$ questions to query so as to construct a confidence interval (CI) for $\theta$ that is as narrow as possible, while maintaining asymptotically frequentist coverage?

Here, coverage (say, $95\%$) means that over repeated runs of the evaluation procedure (i.e., over sampling randomness), the reported $95\%$ CI contains the true finite-bank accuracy $\theta$ at least $95\%$ of the time. This guarantee is crucial for model selection and monitoring: miscalibrated uncertainty can lead to systematic overconfidence and poor deployment decisions, and the risk is amplified under adaptive evaluation where i.i.d.~assumptions and heuristic error bars may fail. 

\textbf{Our contributions}\;\;\; In this paper, we introduce \textit{Factorized Active Querying} (FAQ), an efficient LLM evaluation method that, in our experiments, produces CIs of comparable width using up to $5 \times$ fewer queries than competitive baselines (including uniform sampling), with little to no additional computational cost, while maintaining asymptotic frequentist coverage. FAQ has three components (see Figure \ref{fig:methods} for a full schematic):
\begin{enumerate}[noitemsep, itemsep=0pt, topsep=0pt, parsep=0pt, partopsep=0pt, label=(\alph*)]
    \item A \underline{Bayesian factor model} that extracts information from partially-observed historical data.
    \item A \underline{hybrid sampling policy} combining variance reduction and active learning.
    \item An asymptotic-coverage-preserving finite-population extension of active inference \citep{zrnic2024active}, which we call \underline{Pro-Active Inference (PAI)}. Unlike active inference, PAI directly selects the next question to query in a data-adaptive manner, rather than rigidly proceeding according to a fixed ordering and only making label/skip decisions.
\end{enumerate}
FAQ is designed for repeated evaluation settings, where organizations must compare many candidate or deployed model variants on massive fixed benchmark suites. One group of intended users would be frontier labs monitoring checkpoints (e.g, fine-tunes, distillations, and/or quantized variants). However, a much larger set of organizations whom FAQ is designed for are those performing evaluations who might not be the same ones conducting large-scale pre/post-training. Examples include domain-specific teams (e.g., legal, medical, or scientific organizations), benchmark creators and leaderboard services, and safety/governance teams evaluating many candidate models via paid APIs. In such settings, repeated model evaluation itself can be a meaningful bottleneck in cost (e.g., one evaluation of HELM can cost \$9,337 in API credits, see \citep{li2024active}), latency, and throughput, especially when tasks require multi-turn interactions, simulation/tool use, long-context reasoning, and/or expensive expert verification (e.g., doctors or lawyers).

We evaluate FAQ on two benchmark suites under varying historical-data missingness and quantify the additional gains attributable to PAI. We release our compiled and curated datasets (with $4.4$K+ models and $21.5$K+ questions); datasets and code are available at \url{https://github.com/skbwu/efficiently-evaluating-llms}.

\section{Related Work}

\textbf{Probabilistic factor modeling}\;\;\; FAQ is inspired by probabilistic factor modeling \citep{mnih2007probabilistic,collins2001generalization,chen2023statistical}, which assumes a (Bayesian) generative factor model for a partially-observed binary matrix---here, the matrix of model-question correctness outcomes---and learns low-dimensional latent factors from the data. Joint-estimation approaches include \citep{chen2019joint} and \citep{cui2026identifiability}, which impose additional constraints to support theoretical analysis of latent factor models. Relatedly, \citet{cai2025latent} develop a family-aware latent-variable model for aggregate LLM benchmark performance, incorporating model metadata and benchmark-specific loadings. Importantly, any uncertainty guarantees for these approaches are tied to correct specification of the factor model. In contrast, FAQ uses the Bayesian factor model only as a \textit{mental model} to guide adaptive query selection, while delivering asymptotically valid frequentist coverage regardless of factor-model fidelity. If the data \textit{is} well-described by the factor model, FAQ automatically adapts and delivers tighter CIs.

\textbf{Active statistical inference}\;\;\; For model-free asymptotic frequentist coverage, FAQ builds on active statistical inference \citep{zrnic2024active, gligoric2025can}, which uses machine learning predictions over unlabeled covariates to adaptively select labels under a fixed budget and construct CIs. Our benchmark setting, however, provides no explicit covariates, so we use the factor model to construct latent features to guide adaptive query selection. Moreover, unlike sequential active inference---which, in a pre-specified ``nature's order" stream, can only make label/skip decisions for the next item and enforces the budget constraints only in expectation---PAI/FAQ directly selects the next question from the full benchmark bank using all available information and enforces the query budget exactly.\footnote{While preparing this camera-ready manuscript, we became aware of \citet{hamilton2026active}'s closely related Active Measurement (AM) framework, which also uses online-updated ML predictions to guide adaptive sampling for finite-population inference. Unlike AM, however, FAQ additionally incorporates its online-updated factor-model predictions into PAI's regression adjustment, providing another mechanism for variance reduction beyond adaptive query selection.}

\begin{figure}[!h]
  \centering
  \includegraphics[width=0.95\textwidth]{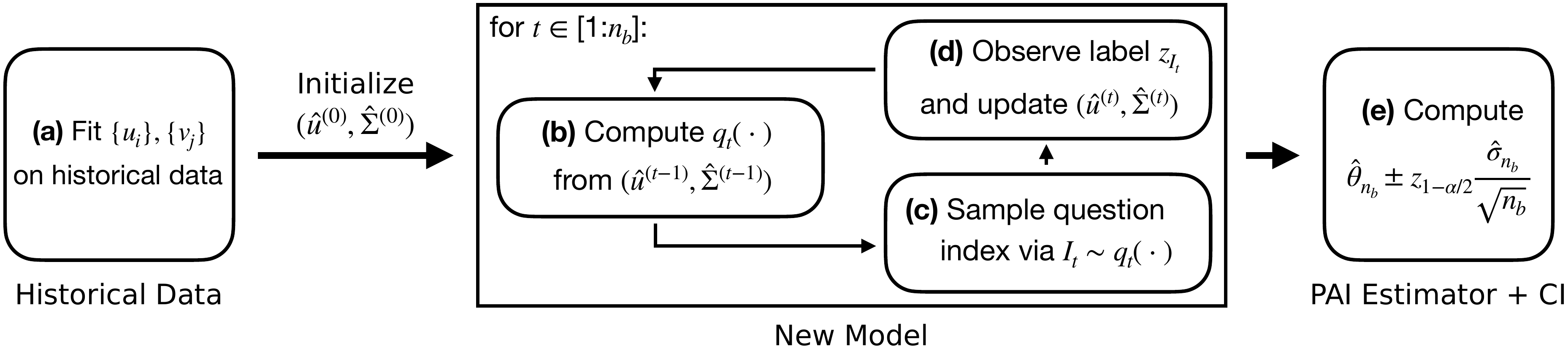}
    \caption{\textbf{High-level overview of Factorized Active Querying (FAQ).} \textbf{(a)} Using partially-observed historical outcomes $H$, we fit historical model and question factors (Eq. \ref{eq:factor-model-objective}). For a new model, we initialize its latent factor from the historical fit. Then, at each sampling round $t$, we \textbf{(b)} compute the hybrid sampling policy $q_t(\cdot)$ from the current factor posterior (Eq. \ref{eq:final-hybrid-sampling-policy}), \textbf{(c)} sample a question $I_t \sim q_t(\cdot)$, and \textbf{(d)} update the model factor using the observed outcome $z_{I_t}$ (Eqs. \ref{eq:bayesian-factor-model-update-1}-\ref{eq:bayesian-factor-model-update-2}). Finally, we \textbf{(e)} output the PAI estimate and confidence interval (Eqs. \ref{eq:proposed-estimator}, \ref{eq:estimator-variance}).}
    \label{fig:methods}
\end{figure}

\textbf{Active learning}\;\;\; FAQ draws on Bayesian active learning ideas \citep[c.f. e.g.,][]{settles2009active,di2023active,mackay1992information,fedorov2013theory,schein2007active} to guide question selection. One component of our hybrid policy upweights questions predicted (under the factor model) to yield larger posterior variance reductions for $\theta$. Because our goal is tight \textit{frequentist} uncertainty quantification of $\theta$, we combine this active-learning component with an oracle-inspired variance-reduction component.

\textbf{LLM evaluations}\;\;\; FAQ contributes to a growing literature on evaluating LLMs using only a fraction of benchmark questions \citep{hardt2025emerging, zhang2025benchmark,vivek2024anchor,polo2024tinybenchmarks,zhong2025efficient,li2024active}. However, most of this line of work targets point estimation of accuracy: with the exception of a simple AIPW estimator considered in \citet{zhang2025benchmark}, \textit{none of these prior works provide any model-free guarantees---on point estimation or coverage.} For example, \citet{li2024active} assume a probabilistic model of response correctness given questions, imposing much stronger assumptions on the benchmark/outcome-generating process than our fixed-bank, model-free setup. Relatedly, \citet{zouhar2025select,zhou2025speeding,perlitz2024efficient} focus on model ranking/selection under limited queries, which is orthogonal to our goal of estimating model accuracy with rigorous uncertainty quantification. \cite{perlitz2024efficient} additionally study benchmark design, which is very different from our goal. Closest in spirit to our work is \citet{angelopoulos2025cost}, which uses active inference to reduce the cost of rating LLM outputs; by contrast, we reduce the number of inference calls per model. More importantly, \citet{angelopoulos2025cost}'s approach posits an i.i.d.~model for the benchmark data and targets the \textit{expected} accuracy as judged by a strong autorater, whereas we estimate the \textit{true, finite-bank} accuracy of each LLM without i.i.d.~assumptions. This distinction is practically important: new models arrive over time and distribution shift is the norm, potentially breaking the i.i.d.~assumption. By targeting finite-bank accuracy, FAQ provides uncertainty quantification that remains meaningful over such drifts. Relatedly, \citet{cowen2026multiple} develop a statistical framework for LLM evaluation leveraging outputs from multiple cheaper models, but do not address adaptive question selection under a fixed evaluation budget. Separately, \citet{blackwell2024towards} document run-to-run variability in LLM evaluation even at zero temperature and with a fixed random seed, highlighting another source of uncertainty in benchmark evaluation.

Because the above approaches generally target different estimands/tasks---point estimation, ranking/selection, benchmark design, or superpopulation accuracy---and generally lack model-free coverage guarantees for each model's finite-bank accuracy, they are not comparable inferential baselines: as such, our empirical baselines (c.f. Section \ref{section:experimental-setup}) comprise procedures that do have such guarantees.

\section{Method}

\subsection{Factor Model: Extracting Information from History}
\label{subsection:factor-model}

To run FAQ, we require features for models and questions, but the partially-observed historical outcome matrix $H$ provides no explicit covariates. Thus, we introduce latent factors---each historical model $i$ has $u_i \in \R^k$ and each question $j$ has $v_j \in \R^k$. We model the likelihood that model $i$ answers question $j$ correctly via a logistic link, where $\sigma(\cdot)$ is the sigmoid function:
{
\begin{align}
    P\left( H_{ij} = 1 \mid \{ u_{i'} \}_{i'=1}^{N_{\text{old}}}, \{ v_{j'} \}_{j'=1}^{N_q} \right) = \sigma\left( u_i^\top v_j \right).
\end{align}}
This model is natural for heterogeneous benchmarks, where performance is driven by a small number of shared capabilities (e.g., coding, reasoning, and domain knowledge): $u_i$ encodes model $i$'s latent proficiencies, $v_j$ the capabilities required by question $j$, and $u_i^\top v_j$ their match. The resulting probabilistic predictions provide uncertainty over unqueried model-question pairs, which we exploit in our hybrid sampling policy. We emphasize that the factor model only serves to improve efficiency: our FAQ coverage guarantees are model-free and do not depend on it.

\textbf{Fitting on partially-observed historical data}\;\;\; We jointly fit the historical model and question factors on the observed entries of $H$ by minimizing the masked negative log-likelihood via AdamW with weight decay $\lambda$: $(\lambda, k)$ are selected via cross-validation (see Section \ref{section:experimental-setup}). The full objective is given in Appendix \ref{appendix:full-log-likelihood}.

\textbf{Deploying to evaluate a new model}\;\;\; At deployment, we fix the historically-learned question factors $\{ v_{j} \}_{j=1}^{N_q}$ as features, since questions do not change across models. For a new model, we initialize $u \sim \mathcal{N}( \uhat^{(0)}, \Sigmahat^{(0)} )$ from the empirical mean and covariance of historical model factors $\{ u_i \}_{i=1}^{N_\text{old}}$. The factor model is only a working approximation for querying and prediction, \textit{not} for Bayesian credible intervals: we make no fidelity assumptions. Asymptotic frequentist validity is ensured by PAI (Section \ref{subsection:finite-pop-active-inference}). 

\textbf{Online Laplace updates}\;\;\; At round $t$, after observing $z_{I_t}$, we update $u \mid \F_t \approx \mathcal{N}(\uhat^{(t)}, \Sigmahat^{(t)})$ via $\Sigmahat^{(t)}, \uhat^{(t)} \leftarrow \texttt{ModelFactorUpdate}(z_{I_t}, v_{I_t}, \uhat^{(t-1)}, \Sigmahat^{(t-1)})$, a computable Laplace update derived and specified in Appendices \ref{appendix:log-reg-derivation} and \ref{appendix:factor-model-online-laplace-updates}. This refines predictions $\phat^{(t)}_j = \sigma( \uhat^{(t)}{}^\top v_{j})$. We adopt this Gaussian approximation for computational tractability, enabling fast, closed-form updates of the new model's latent-factor mean and covariance, which guide our active-learning sampling policy.

\subsection{Pro-Active Inference}
\label{subsection:finite-pop-active-inference}

\textbf{Motivation and desiderata}\;\;\; The factor model is only a working model and offers no coverage guarantee. We also make no i.i.d.~assumptions on benchmark questions: labels $z_j$ are fixed-but-unknown, since curated suites often include adversarial/edge cases and need not represent a superpopulation of user queries. Thus, we need a \textit{model-free} inference layer that wraps around the factor model and yields valid coverage for the finite-bank estimand $\theta$.

Sequential active inference \citep{zrnic2024active} provides coverage, but follows \textit{nature's order}: at time $t$, one observes only the next item in a fixed stream and can only make a label/skip decision via an adaptive Bernoulli policy. If an informative item is skipped, it cannot be revisited. In our benchmark-bank setting, by contrast, we have access to the \textit{full} pool of question factors at each step and wish to \textit{proactively select} the next question, labeling each selected question deterministically.

These desiderata are important for FAQ because the new model factor $u$ is initialized from historical models and may be inaccurate as models evolve. Proactively selecting informative questions early accelerates learning of $u$ (Section \ref{subsection:active-learning-sampling-policy}), potentially improving subsequent query choices and inference. Practically, our desiderata also better match typical evaluation workflows: selecting a budgeted subset of questions, rather than streaming through the entire bank with probabilistic labeling decisions. 

\textbf{Pro-Active Inference estimator}\;\;\; Fix a query budget $n_b$. At each round $t = 1, \dots, n_b$, we sample a question $I_t \in \{ 1, \dots, N_q \}$ from an adaptive distribution $q_t(\cdot)$,\footnote{To support Assumptions \ref{ass:var-stabilize}--\ref{ass:lindeberg} for Theorem \ref{thm:martingale-clt} we require sampling with replacement. Appendix~\ref{appendix:without-replacement} discusses how ad-hoc without-replacement variants can fail to yield valid coverage.} query the new model, and observe label $z_{I_t} \in \{ 0, 1 \}$. We define the \textit{Pro-Active Inference} (PAI) estimator $\thetahat_{n_b}$:
{
\begin{align}
\label{eq:proposed-estimator}
    \thetahat_{n_b} := \frac{1}{n_b}\sum_{t=1}^{n_b} \phi_t,\text{ }\phi_t := \frac{1}{N_q}\sum_{j=1}^{N_q} \phat_{j}^{(t-1)} + \frac{1}{N_q}\frac{z_{I_t} - \phat_{I_t}^{(t-1)}}{q_t\left( I_t \right)}.
\end{align}}
Here, $\phat_j^{(t-1)} := \sigma\left({\uhat^{(t-1)}}{}^{\top} v_j \right)$ and $\uhat^{(t-1)}$ is the current posterior-mean estimate of the new model factor. The first term in $\phi_t$ is the plug-in factor-model estimate of $\theta$, while the second serves as a bias correction. Let $\F_{t-1} := \sigma(\{I_s, z_{I_s}\}_{s=1}^{t-1})$. We require $\phat_j^{(t-1)}$ and $q_t(\cdot)$ to be $\F_{t-1}$-measurable: depending only on past queries/labels. In Theorem~\ref{thm:martingale-clt} (proven in Appendix~\ref{appendix:proof-martingale-clt}), we show that $\thetahat_{n_b}$ is unbiased for $\theta$ and derive its limiting distribution via a martingale central limit theorem (CLT). To formalize asymptotics, we work under a triangular-array setup indexed by $m$. Specifically, we consider a sequence of problems $\{ \mathcal{P}^{(m)} \}_{m \geq 1}$, in which $N_q = N_q^{(m)}$, $n_b = n_b^{(m)}$, and all involved quantities---$\{ z_{1}, \dots, z_{N_q} \}$,
$\theta$, $\thetahat_{n_b}$, $\{ \phat^{(t)} \}_{0 \leq t < n_b}$, and $\{ q_t(\cdot) \}_{1 \leq t \leq n_b}$---may depend on $m$. For readability, we suppress the superscript $(m)$ in the statement and proof.
\begin{theorem}\label{thm:martingale-clt}
Consider a sequence of problems (indexed by $m$) such that $n_b\uparrow \infty$ as $m \rightarrow \infty$ and $n_b \leq N_q$ for every $m$. Then (a) $\thetahat_{n_b}$ is unbiased for $\theta$ for every $m$; and (b) under mild regularity conditions (variance stabilization/control and Lindeberg; Assumptions \ref{ass:var-stabilize}--\ref{ass:pred-prob}), $\thetahat_{n_b}$ obeys $\sqrt{n_b}( \thetahat_{n_b} - \theta)/\sigmahat_{n_b} \toD \mathcal{N}(0, 1)$, where $\sigmahat_{n_b}^2$'s computable closed-form expression is in Appendix \ref{appendix:computable-closed-form-variance}.
\end{theorem}
By Theorem \ref{thm:martingale-clt}, it follows that an asymptotic $(1-\alpha)$-level CI for $\theta$ is $[ \thetahat_{n_b} \pm z_{1 - \alpha / 2} (\sigmahat_{n_b} / \sqrt{n_b}) ]$. For a fixed bank of outcomes $(z_1, \ldots, z_{N_q})$, the coverage probability of this interval is
\begin{align*}
    P_{\text{sampling}}\left( \theta \in \left[ \thetahat_{n_b} \pm z_{1 - \alpha / 2} \frac{\sigmahat_{n_b}}{\sqrt{n_b}} \right] \mid z_1, \ldots, z_{N_q} \right),
\end{align*}
where randomness arises solely from adaptive question selection. Although the theorem is stated for a sequence of problem instances, we observe a single one in practice. Consequently, Theorem~\ref{thm:martingale-clt} implies that the above CI attains coverage at approximately the nominal level $1 - \alpha$ for sufficiently large problem sizes.\footnote{Nonasymptotic confidence intervals can also be constructed using martingale concentration inequalities, but straightforward approaches can be prohibitively conservative. For example, applying Azuma-Hoeffding under our hybrid sampling policy (c.f. Section \ref{sec:practical-hybrid-sampling-policy}) with the default uniform-mixing weight of $\tau = 0.05$ and a budget of $n_b = 1,200$ labels, yields a $95\%$ confidence interval with half-width approximately $1.57$, exceeding the entire parameter range $[0, 1]$. Sharper, potentially variance-adaptive nonasymptotic bounds may be possible, but their development is beyond the scope of this work. We also empirically assess FAQ's finite-budget coverage in Section \ref{sec:empirical-results}.}

\subsection{An Oracle-Optimal Sequential Sampling Policy}
\label{subsection:oracle-optimal}

In practice, constructing a sequential policy $q_t(\cdot)$ that yields narrow CIs is difficult, partly because $\phat_j^{(t)}$ evolves over rounds $t$. We therefore derive guidance from an oracle idealization. For this subsection only, assume an oracle superpopulation model: (a) $z_j \sim \text{Bern}(p_j)$ independently across questions $j$;
(b) the oracle has access to the true, time-invariant $\{ p_j \}$; and (c) the target is $\theta^* := N_q^{-1} \sum_{j=1}^{N_q} p_j$.
As a corollary of Theorem \ref{thm:martingale-clt}(a), $\thetahat_{n_b}$ is unbiased for $\theta^*$. A natural direction is to minimize its variance in the oracle world (proof in Appendix~\ref{appendix:proof-oracle-min-var}):\footnote{While access to $\{ p_j \}$ makes $\theta$ known, we wish to find the variance-minimizing policy among estimators of form $\thetahat_{n_b}$.}
\begin{theorem}
    \label{thm:oracle-min-var} In the aforementioned oracle setup, the variance-minimizing policy $q_t^*(j)$ is time-independent and uniquely defined by $q_t(j) \propto \sqrt{p_j (1-p_j)}$.
\end{theorem}
\subsection{An Active-Learning-Driven Sampling Policy}
\label{subsection:active-learning-sampling-policy}

In practice, we use the factor model to estimate question difficulty as $\phat_j^{(t-1)} := \sigma({\uhat^{(t-1)}}{}^\top v_j)$. Since question factors $\{ v_j \}$ are fixed, prediction quality depends on how well we have learned $\uhat^{(t-1)}$---motivating an active-learning component that quickly learns $u$, but only insofar as it improves inference for $\theta^*$. For this subsection only, assume the factor model is correctly specified: (a) $z_{j} \sim \text{Bern}\left( p_j \right)$ with $p_j := \sigma\left( u^\top v_{j} \right)$; (b) $u$ has a normal prior; and (c) $\theta^*$ is random only through $u$:\footnote{We emphasize that the factor model is only a \emph{mental model} that helps us in deriving a reasonable adaptive sampling policy.}
\begin{align*}
    \theta^* = \frac{1}{N_q} \sum_{j=1}^{N_q} p_j = \frac{1}{N_q} \sum_{j=1}^{N_q} \sigma\left( u^\top v_{j} \right) =: g(u).
\end{align*}
Given $u \mid \F_{t-1} \sim \mathcal{N}\left( \uhat^{(t-1)}, \Sigmahat^{(t-1)} \right)$, we seek the query that most reduces the posterior variance of $g(u)$ (c.f. \citet{mackay1992information}). To obtain a closed-form expression, we approximate the updated posterior $u \mid \F_{t-1}, z_j$ as normal and linearize $g$ around $\uhat^{(t-1)}$. Under these approximations, the index $I_t$ that maximally reduces $\theta^*$'s posterior variance is given by $I_t = \argmax_j d^{(t)}(j)$, whose computable closed-form expression is presented in Appendix \ref{appendix:computable-closed-form-posterior-variance} and derived in Appendix~\ref{appendix:proof-theta-glm-variance-reduction}.

\subsection{A Practical Hybrid Sampling Policy}
\label{sec:practical-hybrid-sampling-policy}

In practice, we trade off exploitation (variance-reduction using the current $\phat^{(t-1)}$) and exploration (querying informative questions to quickly refine $\uhat^{(t)}$ and hence $\phat^{(t)}$). We implement this via a time-varying hybrid sampling policy. Inspired by Sections \ref{subsection:oracle-optimal}--\ref{subsection:active-learning-sampling-policy}, at time $t$, define unnormalized oracle (o) and active-learning (a) scores as follows: $s^{(t)}_{\text{o}}(j) := \sqrt{\phat_j^{(t-1)} ( 1 - \phat_j^{(t-1)} )}, \; s^{(t)}_{\text{a}}(j) := d^{(t)}(j).$
Define $\text{Norm}(x)(j) = x(j) / \sum_{j'} x(j')$. We convert each score to a probability vector via $h^{(t)}_{\text{o}} = \text{Norm}(s^{(t)}_{\text{o}})$ and $h^{(t)}_{\text{a}} = \text{Norm}(s^{(t)}_{\text{a}})$.
Given hyperparameters $\rho, \gamma, \beta_0, \tau \in (0, 1)$, set mixing weight $\alpha_t = \max\left( 0, 1 - \frac{t}{\rho n_b} \right)$ and tempering exponent $\beta_t = \beta_0 \min\left( 1, \frac{t}{\gamma n_b}\right)$.
We mix the normalized scores and apply tempering to form our combined-and-tempered score $h_{\text{cat}}^{(t)}$:
{\begin{align}
    \label{eq:combined-and-tempered-score}
    h_{\text{cat}}^{(t)}(j) = \left( \left( 1 - \alpha_t \right) h^{(t)}_{\text{o}}(j) + \alpha_t h^{(t)}_{\text{a}}(j) \right)^{\beta_t}.
\end{align}}
Finally, we normalize $h_{\text{cat}}^{(t)}(\cdot)$ to be a probability vector and apply $\tau$-uniform mixing to lower-bound all sampling probabilities. Our final hybrid policy $q_t(\cdot)$ is given by:
{\begin{align}
    \label{eq:final-hybrid-sampling-policy}
    q_t(j) = \frac{\tau}{N_q} + \left( 1 - \tau \right) \text{Norm}(h_{\text{cat}}^{(t)})(j).
\end{align}}
Here, $\alpha_t$ sets the exploration-to-exploitation schedule (decaying to $0$ after $\rho n_b$ queries), while $\beta_t$ tempers early scores towards near-uniform exploration (i.e., scores closer to each other) and reaches $\beta_0$ after $\gamma n_b$ queries. The tempering and $\tau$-mixing ensure that all questions retain nontrivial selection probability---supporting Theorem \ref{thm:martingale-clt}'s regularity conditions.

We initialize a new model's latent factor as $u \sim \mathcal{N}( \uhat^{(0)}, \Sigmahat^{(0)} )$. For rounds $t=1, \dots, n_b$, FAQ:
\begin{enumerate}[noitemsep, topsep=0pt, parsep=0pt, label=(\alph*)]
    \item Updates $q_t(\cdot)$ under the current factor-model and samples $I_t \sim q_t(\cdot)$ (Eq. \ref{eq:final-hybrid-sampling-policy}).
    \item Queries the model to observe $z_{I_t}$.
    \item Updates the PAI estimate (Eq. \ref{eq:proposed-estimator})
    \item Updates $(\uhat^{(t)}, \Sigmahat^{(t)})$ (Eqs. \ref{eq:bayesian-factor-model-update-1}--\ref{eq:bayesian-factor-model-update-2}).
\end{enumerate}
After $n_b$ rounds, we estimate the variance and form a CI (Eq. \ref{eq:estimator-variance}). Algorithm \ref{alg:full-faq} gives pseudocode.

\begin{algorithm}[!h]
  \caption{Factorized Active Querying}
  \label{alg:full-faq}
  \begin{algorithmic}[1]
    \STATE {\bfseries Input:} Historical data matrix $H$, confidence level $\alpha$, and budget $n_b$
    \STATE Fit factor model on $H$ to initialize prior $\mathcal{N}(\hat{u}^{(0)}, \hat{\Sigma}^{(0)})$ for new model's latent factor.
    \FOR{$t=1$ {\bfseries to} ${n_b}$}
    \STATE Compute hybrid probability vector over queries at time $t$: \[q_t(j) = \frac{\tau}{N_q} + (1-\tau)\mathrm{Norm}(h^{(t)}_{\mathrm{cat}})(j),\] where $h^{(t)}_{\mathrm{cat}}$ is the tempered convex combination of oracle variance-reduction and active-learning-driven policies given in Equation \eqref{eq:combined-and-tempered-score}.
    \STATE Sample new index $I_t \sim q_t(\cdot)$ and query LLM with question $I_t$ to obtain $z_{I_t}$.
    \STATE Compute \[\phi_t := \frac{1}{N_q}\sum_{j=1}^{N_q}\hat{p}_j^{(t-1)} + \frac{1}{N_q}\cdot\frac{z_{I_t}-\hat{p}_{I_t}^{(t-1)}}{q_t(I_t)}.\]
    \STATE Update factor model prior via \[(\uhat^{(t)}, \Sigmahat^{(t)}) \gets (\texttt{ModelFactorUpdate}(z_{I_t}, v_{I_t}, \uhat^{(t-1)}, \Sigmahat^{(t-1)}),\] given in Equations~\eqref{eq:bayesian-factor-model-update-1}--\eqref{eq:bayesian-factor-model-update-2}.
    \ENDFOR
    \STATE Compute the PAI variance estimate given in Equation~\eqref{eq:estimator-variance}: \[\sigmahat_{n_b}^2 = \frac{1}{{n_b} N_q^2} \sum_{t=1}^{n_b} \frac{\left( z_{I_t} - \phat_{I_t}^{(t-1)}\right)^2}{q_t(I_t)^2} - \frac{1}{{n_b} N_q^2} \sum_{t=1}^{n_b} \left( \frac{1}{{n_b}}\sum_{s=1}^{n_b} \frac{z_{I_s}}{q_s(I_s)} - \sum_{j=1}^{N_q} \phat_j^{(t-1)} \right)^2,\] using samples $\{z_{I_t}\}_{t=1}^{n_b}$, computed probability vectors $\{q_t(\cdot)\}_{t=1}^{n_b}$, and factor model estimates $\{\hat{p}_j^{(t)}\}_{j \in [N_q], t \in [0:{n_b-1}]}$.
    \STATE \textbf{Return} Estimate $\hat{\theta}_{n_b} := \frac{1}{{n_b}}\sum_{t=1}^{n_b}\phi_t$ and $(1-\alpha)$-CI: $\left[\hat{\theta}_{n_b} \pm z_{1-\alpha/2}\cdot \frac{\sigmahat_{n_b}}{\sqrt{{n_b}}}\right].$
    \end{algorithmic}
\end{algorithm}

\section{Experimental Setup}
\label{section:experimental-setup}

\textbf{Datasets}\;\;\; Using the HuggingFace \texttt{datasets} API, we collect per-question 0/1 outcomes for $4.4$K models from the Open LLM Leaderboard \citep{open-llm-leaderboard-v2, myrzakhan2024open}. Each model is evaluated on $21.6$K questions across six datasets, which we aggregate into two suites: (a) MMLU-Pro with $12$K questions; and (b) BBH+GPQA+IFEval+MATH+MuSR with $9.5$K questions. Full dataset citations are in Appendix \ref{appendix:experimental-details}. We release our compiled datasets to support future research. We focus on these suites because they provide large, reusable model-by-unit outcome matrices; interactive agent benchmarks with task-level success outcomes are natural targets for FAQ/PAI, but running them at comparable scale would require substantially greater evaluation infrastructure.

\textbf{Historical splits and missingness}\;\;\; For each suite, we sort models by release date, take the first $2.2$K as historical $H$ and the next $2.2$K as new test models. We induce missingness in $H$ by selecting $n_{\text{full-obs}}$ fully-observed rows uniformly at random, and independently masking each remaining entry with probability $1 - p_{\text{obs}}$ (i.e., MCAR: missing completely at random). Details are in Appendix \ref{appendix:experimental-details}.

\textbf{Factor model selection}\;\;\; For each non-cold-start missingness setting, we split $H$ by rows (models) into $H_{\text{train}}$ (first $80$\%) and $H_{\text{val}}$ (last $20\%$). We choose $(\lambda, k)$ via 5-fold cross-validation on $H_{\text{train}}$: (a) fit on observed entries (Eq. \ref{eq:factor-model-objective}) and evaluate binary prediction accuracy on an MCAR-masked holdout set; and (b) apply the 1-SE rule \citep{friedman2010}.

\textbf{Budgets and coverage}\;\;\; We consider ten evenly-spaced budgets $n_b \in [2.5\%, 25\%]$ of the question bank and target $95\%$ nominal coverage for all experiments.

\textbf{FAQ hyperparameters}\;\;\; For each non-cold-start budget/missingness setting, we tune $(\rho, \gamma, \beta_0, \tau)$ to minimize mean CI width on $H_{\text{val}}$ (averaged over 5 seeds), using a factor model refit on all of $H_{\text{train}}$.\footnote{On $H_\text{val}$, $(\beta_0, \tau) = (1.0, 0.05)$ is a robust default across budget/missingness settings. $(\rho, \gamma)$ are more setting-dependent.} For tuning, we run FAQ for $n_b$ rounds per validation model, querying selected questions when labels are missing. This mirrors deployment, where tuning expends additional queries on held-out models, beyond the reported query budget for evaluating each new model.

\textbf{Baselines}\;\;\; Let $\pbar_j$ be the mean of the \textit{observed} outcomes in $H$ for question $j$\footnote{If question $j$ is never observed in $H$ (e.g., when $0.1\%$ entries in $H$ are observed), we impute $\pbar_j$ with $H$'s global mean.}, used as a proxy for question difficulty. For each budget/missingness setting, we compare against model-assisted, design-based AIPW estimators for finite-population inference, using established sequential sampling schemes from survey sampling and active inference. Unlike FAQ/PAI, these baselines follow \textit{nature's order}: questions arrive in a fixed stream and question $t$ is labeled with probability $\pi_t(t)$. Let $\xi_t \sim \text{Bern}(\pi_t(t))$ indicate labeling. Each baseline uses
{\begin{align}
\label{eq:baseline-aipw}
\thetahat_{\text{base}}
=
\frac{1}{N_q}\sum_{t=1}^{N_q}
\left(
p_{\text{base}}(t)
+
\left( z_{t}-p_{\text{base}}(t)\right)\frac{\xi_t}{\pi_t(t)} \right),
\end{align}
with baseline predictor $p_{\text{base}}(t) \in \{ 0, \pbar_t \}$. We consider: (a) $\pi_t(t) = n_b / N_q$ (Bernoulli); (b) $\pi_t(t) \propto \sqrt{\pbar_t (1 - \pbar_t)}$ (Neyman); and (c) $\pi_t(t) \propto \min(\pbar_t, 1-\pbar_t)$ (c.f. \citet{zrnic2024active}). For (b)-(c), we normalize, apply $\tau$-mixing, and use the budget-stabilization wrapper in Equation 8 of \citet{zrnic2024active}. Note that $p_{\text{base}}(t) = 0$ and $\pi_t(t) = n_b / N_q$ reduces to \textit{uniform sampling.}}

\textbf{Traditional active inference ablation}\;\;\; Since the baselines above omit the factor model---a core component of FAQ---we additionally ablate FAQ to isolate the contribution of PAI and our factor-model-guided hybrid sampling policy. We replace PAI and hybrid sampling with \citet{zrnic2024active}'s sequential active inference---keeping the same factor-model predictions $\phat_j^{(t)}$, using their suggested labeling policies (Eqs.~3 and 8 therein), and tuning their $\tau$-mixing parameter (Appendix \ref{appendix:experimental-details}).

\begin{figure}[!t]%
  \centering%
  \includegraphics[width=1.0\textwidth]{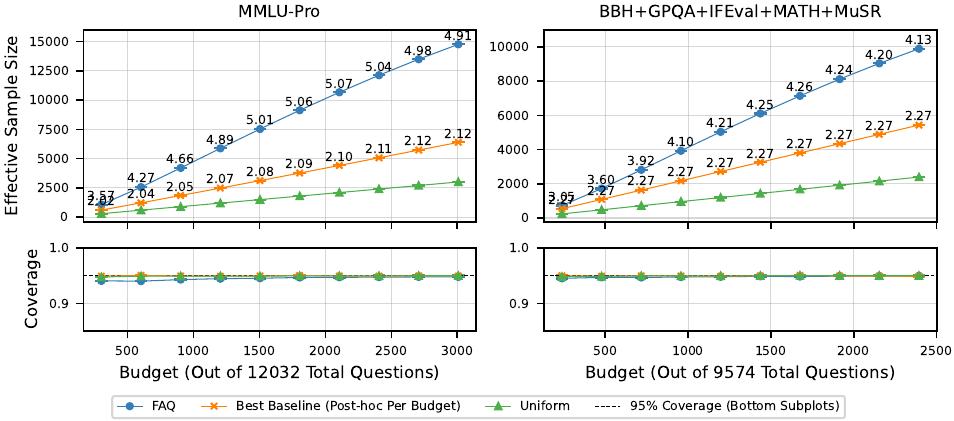}%
    \caption{\textbf{Evaluating $2.2$K LLMs on MMLU-Pro and BBH+GPQA+IFEval+MATH+MuSR with fully-observed historical data.} \textbf{(Top Row)} Effective sample size (ESS) vs. budget for FAQ, the strongest baseline per budget (selected post-hoc, favoring baselines), and uniform sampling. Numbers annotate the ESS multiplier relative to uniform sampling. For example, on MMLU-Pro at budget $1500$, an ESS multiplier of $5.01$ means FAQ matches uniform sampling's CI width at budget $7515$ (a $5.01\times$ gain). \textbf{(Bottom Row)} Empirical coverage of the corresponding $95\%$ normal-approximation CIs, averaged over $2.2$K test models and $100$ seeds. Standard errors are shown, but negligible on the plot scale. \textbf{Higher ESS and ESS-multiplier values indicate better performance.}}%
    \label{fig1:ess+coverage_fully-observed}%
\end{figure}%
\begin{figure}[!b]%
  \centering%
    \includegraphics[width=0.95\textwidth]{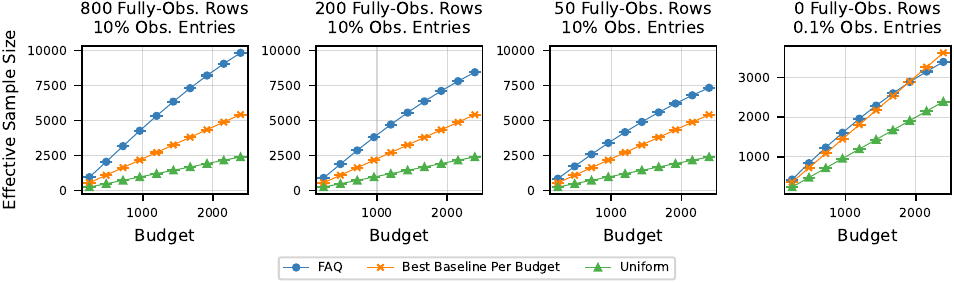}%
    \caption{\textbf{ESS under missing historical data on BBH+GPQA+IFEval+MATH+MuSR.} ESS vs.~budget across missingness settings (Section \ref{section:experimental-setup}). Standard errors are shown, but negligible on the plot scale. MMLU-Pro and coverage results are in Appendix \ref{appendix:additional-results}.}%
    \label{fig2_ess_sparsity-settings}%
\end{figure}%
\textbf{Metrics}\;\;\; We evaluate (a) coverage and (b) effective sample size (ESS). For each setting, we run $100$ random seeds for each of $2.2$K test models. In each seed/model, methods output Wald intervals of the form $[ \thetahat_{\text{method}} \pm 1.96 \sigmahat_{\text{method}} ]$ (clipped to be in $[0, 1]$), where $\sigmahat_{\text{method}}$ is computed via the method's plug-in variance formula at budget $n_b$. Coverage is computed as the fraction of intervals containing $\theta$, averaged over models and seeds. ESS at budget $n_b$ is defined as
$n_{\text{eff}} :=(\sigmahat^2_{\text{uniform}}/\sigmahat^2_{\text{method}}) n_b,
$ where $\sigmahat_{\text{uniform}}^2$ is the plug-in variance of uniform sampling at the same budget $n_b$. Intuitively, $n_{\text{eff}}$ is the uniform-sampling budget that would yield the same variance as the method at budget $n_b$: thus, larger ESS implies stronger performance. We emphasize ESS (over raw interval width) since it directly corresponds to query-cost savings: an ESS multiplier of $m$ means matching uniform sampling's CI width with roughly a $1/m$ fraction of the queries.\footnote{As all tested methods are unbiased and asymptotically normal, ESS gains translate into proportional reductions in MSE.}

\textbf{Good-faith evaluation}\;\;\; FAQ is tuned using only historical train/validation data (without access to test models). In contrast, to avoid understating baseline performance, we report baselines \textit{after post-hoc selection}: for each budget/dataset/missingness setting, we sweep baseline/ablation variants/hyperparameters on the test models (100 seeds) and report the best configuration (averaged across seeds), yielding an oracle-style upper bound that favors the baselines.

\textbf{Cold-start and cross-suite transfer}\;\;\; While FAQ as presented above leverages target-suite historical data to fit the factor model, practitioners may ask how it performs on a new benchmark before such data have accumulated. We therefore explore a cold-start factor-model experiment on GPQA (All), initializing factors via MMLU-Pro transfer: question factors use weighted $k$-NN over LLM-embedding cosine similarities of question text, while model factors come from the MMLU-Pro source fit. Because no target-suite historical matrix is available for factor-model selection, we use a fixed a priori factor-model configuration. FAQ sampling hyperparameters are tuned analogously on held-out validation models and minimizing CI width only. Full details are in Appendix \ref{appendix:experimental-details}.

Hyperparameter search ranges are in Appendix \ref{appendix:experimental-details}. Full results (including widths) are in Appendix \ref{appendix:additional-results}.

\section{Empirical Results}
\label{sec:empirical-results}

Our primary experiments compare FAQ's ESS and coverage against uniform sampling and post-hoc-selected baselines across two benchmark suites and varying historical-data missingness levels (Figures \ref{fig1:ess+coverage_fully-observed}--\ref{fig2_ess_sparsity-settings}). We also audit FAQ's per-model coverage under realistic deployment shifts (Figure \ref{fig3_per-model-coverages-7.5-bbh}) and explore a cold-start transfer without target-suite historical data for factor-model fitting (Figure \ref{figx:cold-start}). A factor-model-matched active-inference ablation is summarized below, with full results in Appendix~\ref{appendix:traditional-active-inference-ablation}. Per setting, one parallelized FAQ run over $2.2$K models and $12$K questions takes under one minute on a single H100 GPU: negligible compared to model inference/evaluation time.\footnote{Additional exploratory experiments examine structured historical-data missingness, historical-label corruption, and sensitivity to factor-model predictive accuracy at selected query budgets (Appendix \ref{appendix-exploratory-experiments}). These results further support FAQ's empirical efficiency gains and near-nominal coverage under a range of perturbations.}

\textbf{Fully-observed historical data}\;\;\; With fully-observed $H$, FAQ achieves: (a) $4$--$5\times$ higher ESS than uniform sampling; (b) $1.8$--$2.4\times$ higher ESS than the strongest post-hoc baseline; and (c) empirical coverage near the $95\%$ target (Figure \ref{fig1:ess+coverage_fully-observed}). These gains are stable across budgets and translate to roughly proportional query-cost savings. At the smallest budgets, FAQ shows mild ESS degradation and miscoverage under $1\%$---likely from noisy early factor estimates and less accurate CLT normal approximations, respectively.\footnote{Future work could explore $t$- or bootstrap-calibrated (over the martingale increments) quantiles.} Even there, FAQ still considerably outperforms all baselines.

\begin{figure}[!t]%
  \centering%
    \includegraphics[width=0.95\textwidth]{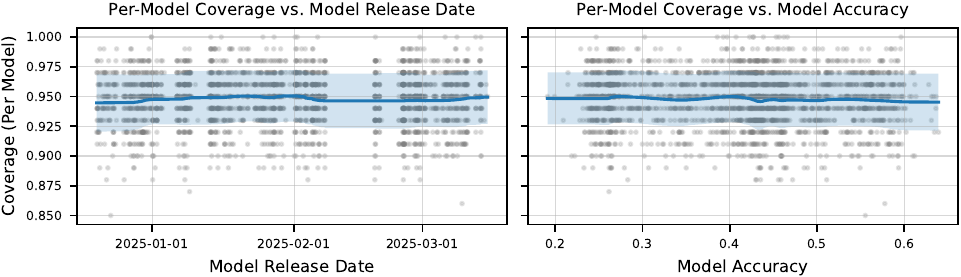}%
    \caption{\textbf{Per-model coverages on BBH+GPQA+IFEval+MATH+MuSR with fully-observed historical data at budget $7.5\%$.} Per-model coverages (averaged over $100$ seeds) vs. \textbf{(Left)} model release date and \textbf{(Right)} true accuracy. Gray dots: individual models; solid-blue curve: locally-weighted mean with a $K=501$ nearest-neighbor Gaussian kernel; light-blue band: local $\pm 1$ SD.}%
    \label{fig3_per-model-coverages-7.5-bbh}%
\end{figure}
\textbf{Partially-observed historical data}\;\;\; Since fully-observed history is rarely available at scale, we evaluate FAQ under progressively sparser $H$. Figure \ref{fig2_ess_sparsity-settings} shows that even with $12$--$42\%$ of $H$ observed (e.g., $n_{\text{full-obs}} \in \{50, 800 \}$ with $p_{\text{obs}} = 0.1$ on remaining rows), FAQ substantially outperforms the strongest post-hoc baselines. Even with only $0.1\%$ of historical entries observed, FAQ retains a modest but consistent ESS edge (Figure \ref{fig2_ess_sparsity-settings}, last column) except at the two largest budgets, where noisy factors learned under such extreme sparsity can be over-trusted as tempering weakens late-budget.

\textbf{Coverage audits under deployment shifts}\;\;\; Figure \ref{fig1:ess+coverage_fully-observed} reports coverage aggregated over $2.2$K models and $100$ seeds. For deployment, it is also important that coverage does not vary systematically across models---e.g., that it is not driven by ``easy" extreme-accuracy models while under-covering moderate, high-uncertainty ones. This is especially relevant here because we learn question factors from historical data and then freeze them, initializing new model factors from the historical mean/covariance. As newer models drift from the historical cohort (e.g., via new architectures or training recipes), these factors may become less representative, potentially impacting finite-sample coverage. Consistent with PAI's guarantee of valid coverage for every evaluated model, Figure \ref{fig3_per-model-coverages-7.5-bbh} shows stable coverage across model release date and model accuracy, with no systematic deviations.

\begin{figure}[!t]%
  \centering%
  \includegraphics[width=0.9\textwidth]{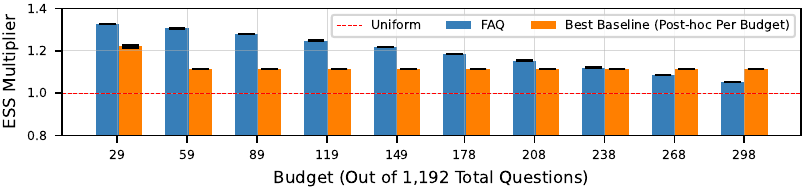}%
    \caption{\textbf{Cold-start FAQ on GPQA (All).} ESS multipliers relative to uniform sampling across budgets for cold-start FAQ \textit{with zero target-suite historical data for factor-model fitting} vs. strongest post-hoc baseline per budget \textit{with access to $0.1\%$ target-suite historical data.} Standard errors are shown, but negligible on the plot scale. \textbf{Higher ESS-multiplier values indicate better performance.}}%
    \label{figx:cold-start}%
\end{figure}

\textbf{Cold-start and cross-suite transfer}\;\;\; FAQ, \textit{with zero target-suite historical data for factor-model fitting}, achieves ESS multipliers up to $1.32\times$ relative to uniform sampling and outperforms the strongest post-hoc selected baselines \textit{with access to $0.1\%$ target-suite historical data} on 8 of 10 budgets (Figure \ref{figx:cold-start}; coverages in Figure \ref{figx:cold-start_appendix}). This suggests that cross-benchmark structure transferred through LLM embeddings can make FAQ useful even before target-suite history accumulates, providing a natural onboarding path for continuous monitoring. 

\textbf{Active-inference ablation}\;\;\; Appendix~\ref{appendix:traditional-active-inference-ablation} compares FAQ to a factor-model-matched sequential active-inference ablation. FAQ yields narrower CIs in the low-budget regime most relevant for query-cost savings, while the sequential ablation can match or outperform FAQ at larger budgets.

\section{Discussion and Conclusion}

We introduced FAQ, a finite-bank framework for LLM evaluation that, in our experiments, achieves up to $5\times$ higher ESS than uniform sampling and $2.4\times$ higher ESS than the strongest baselines, with negligible computational overhead and model-free asymptotic frequentist coverage. FAQ also supports \textit{continuous monitoring}: each evaluation leverages accumulated historical data while contributing fresh outcomes, making uncertainty-quantified evaluation progressively more cost-effective. Although hyperparameter tuning incurs additional query costs on held-out validation models (c.f. Section \ref{section:experimental-setup}), this one-time cost per evaluation setting can be amortized across repeated model evaluations. Our empirically robust default of $(\beta_0, \tau) = (1.0, 0.05)$ further reduces the need for extensive tuning. This is particularly valuable when repeatedly evaluating model variants (e.g., checkpoints, fine-tunes, distillations, quantizations, or deployed systems) under limited query, latency, or adjudication budgets, where even modest per-run savings can compound.

Our limitations suggest future work. FAQ/PAI offers a framework for statistically rigorous, budget-constrained benchmarking: flexible predictions guide query selection, while PAI preserves asymptotically valid finite-bank inference. Our simple factor-model instantiation illustrates the core idea, but is not fundamental: FAQ’s coverage is model-free, and any probability-outputting predictive model and adaptive sampling policy satisfying PAI’s regularity conditions would retain valid coverage. Updating question factors after deployment waves, or using drift-aware priors informed by metadata, embeddings, or other side information, could improve efficiency without changing validity. For example, hierarchical latent-variable models could capture model-family and question-domain effects, while more flexible predictors based on Gaussian processes, neural networks, or LLMs could leverage additional metadata, question text, and embeddings. In the complete absence of historical data, transferable benchmark information, or other predictive side information, FAQ cannot be expected to improve upon uniform sampling. Nevertheless, PAI reduces exactly to classical uniform sampling by using constant predictions and uniform query probabilities, retaining its inferential guarantees.

FAQ also applies to any fixed benchmark bank with binary task-completion outcomes, including multi-turn or agentic benchmarks summarized by task success rates. Extending FAQ to non-binary feedback and batched querying would broaden applicability and parallelize evaluation across adjudicators/inference workers. 

\newpage

\section*{Acknowledgments and Disclosure of Funding}

We thank Kevin Luo and Joon Lee for their many insightful comments and helpful discussions. S.W. and Y.N. were both supported by National Science Foundation Graduate Research Fellowships. E.J.C. was supported by the Office of Naval Research grant N000014-20-1-2157. We also gratefully acknowledge support from Arena Intelligence’s Academic Partnerships Program. This work used Marlowe \citep{marlowe2025}, Stanford University's GPU-based Computational Instrument, supported by Stanford HAI and Stanford Research Computing. This work also used Stanford University's Sherlock computing cluster. We thank Stanford University and the Stanford Research Computing Center for providing the computational resources and support that contributed to these research results. We also thank Scott Linderman for access to his H100 GPUs used in some experiments. The authors have no competing interests to declare.

\bibliography{references}
\bibliographystyle{unsrtnat}
\newpage

\appendix

\section{Proofs, Algorithms, and other Formulas}
\label{appendix:proofs}

\subsection{Full Historical Factor Model Negative Log-Likelihood Objective}
\label{appendix:full-log-likelihood}

Let $O$ be a 0/1-valued $N_{\text{old}} \times N_q$ matrix such that $O_{ij} = 1$ if $H_{ij}$ is observed, else $O_{ij} = 0$. Our historical minimization objective is given by the following:
\begin{align}
\label{eq:factor-model-objective}
    L\left( \{ u_{i'} \}_{i'=1}^{N_{\text{old}}}, \{ v_{j'} \}_{j'=1}^{N_q} \right) = -\sum_{i=1}^{N_{\text{old}}} \sum_{j=1}^{N_q} O_{ij} \left( H_{ij} \log \sigma\left( u_i^\top v_j \right) + \left( 1 - H_{ij} \right)\log\left( 1 - \sigma\left( u_i^\top v_j \right) \right)\right),
\end{align}
We implement using PyTorch's binary cross-entropy function for numerical stability. In practice, $\{ u_{i'} \}_{i'=1}^{N_{\text{old}}}$ and $\{ v_{j'} \}_{j'=1}^{N_q}$ are fit using AdamW with weight decay $\lambda$ for L2 regularization, with $(\lambda, k)$ chosen via cross-validation (Section \ref{section:experimental-setup}).

\subsection{Formulas for Factor Model Online Laplace Updates, Variance Estimators, and Hybrid Sampling Policy Components}
\label{appendix:full-formulas}

\subsubsection{Factor Model Online Laplace Updates}
\label{appendix:factor-model-online-laplace-updates}

At round $t$, after querying question $I_t$ and observing $z_{I_t}$, we update the Gaussian approximation $u \mid \F_t \approx \mathcal{N}(\uhat^{(t)}, \Sigmahat^{(t)})$ via the following online Laplace update (derivation in Appendix \ref{appendix:log-reg-derivation}), denoting these updates as $\Sigmahat^{(t)}, \uhat^{(t)} \leftarrow \texttt{ModelFactorUpdate}(z_{I_t}, v_{I_t}, \uhat^{(t-1)}, \Sigmahat^{(t-1)})$: 
{
\begin{align}
    \label{eq:bayesian-factor-model-update-1}
    \phat^{(t-1)}_{I_t} := \sigma\left( \uhat^{(t-1)}{}^\top v_{I_t} \right), \hat{w} := \phat^{(t-1)}_{I_t} (1 - \phat^{(t-1)}_{I_t}),\\
    \label{eq:bayesian-factor-model-update-2}
    \Sigmahat^{(t)} := \Sigmahat^{(t-1)} - \frac{\hat{w} \Sigmahat^{(t-1)} v_{I_t}v_{I_t}^\top \Sigmahat^{(t-1)}}{1 + \hat{w} v_{I_t}^\top \Sigmahat^{(t-1)} v_{I_t}}, 
    \uhat^{(t)} := \uhat^{(t-1)} + \Sigmahat^{(t)} \left( z_{I_t} - \phat^{(t-1)}_{I_t} \right) v_{I_t}.
\end{align}}

\subsubsection{Computable Closed-Form Variance Estimator for Pro-Active Inference}
\label{appendix:computable-closed-form-variance}
The estimator for the limiting variance of $\thetahat_{n_b}$ in Theorem \ref{thm:martingale-clt} is as follows (derived in Appendix \ref{appendix:proof-martingale-clt}):
\begin{equation}
    \label{eq:estimator-variance}
    \sigmahat^2_{n_b} := \frac{1}{n_b N_q^2} \sum_{t=1}^{n_b} \frac{\left( z_{I_t} - \phat_{I_t}^{(t-1)}\right)^2}{q_t(I_t)^2} - \frac{1}{n_b N_q^2} \sum_{t=1}^{n_b} \left( \frac{1}{n_b}\sum_{s=1}^{n_b} \frac{z_{I_s}}{q_s(I_s)} - \sum_{j=1}^{N_q} \phat_j^{(t-1)} \right)^2.
\end{equation}

\subsubsection{Computable Closed-Form Expression for Maximal Posterior-Variance-Reducing Index under Assumptions and Approximations in Section \ref{subsection:active-learning-sampling-policy}.}
\label{appendix:computable-closed-form-posterior-variance}

Under the assumptions and approximations in Section \ref{subsection:active-learning-sampling-policy}, the index $I_t$ that maximally reduces $\theta^*$'s posterior variance is given by $I_t = \argmax_j d^{(t)}(j)$ (derived in Appendix \ref{appendix:proof-theta-glm-variance-reduction}), where
\begin{align*}
        d^{(t)}(j) = \frac{\hat{w}^{(t-1)}_j \left(v_j^\top \Sigmahat^{(t-1)} \frac{1}{N_q} \sum_{j'=1}^{N_q} \hat{w}^{(t-1)}_{j'} v_{j'} \right)^2}{1 + \hat{w}^{(t-1)}_j v_j^\top \Sigmahat^{(t-1)} v_j},\\
        \hat{w}^{(t-1)}_j = \sigma\left( \uhat^{(t-1)}{}^\top v_j \right) \left( 1 - \sigma\left( \uhat^{(t-1)}{}^\top v_j \right) \right).
\end{align*}

\subsection{Derivation of Online Logistic Regression Laplace Approximation Update}
\label{appendix:log-reg-derivation}

\textit{This derivation was inspired by Ch. 4 of \cite{bishop2006pattern}'s discussion on Laplace approximations and Bayesian logistic regression. The key idea is to Taylor expand the log-posterior around the current mean, which yields approximate closed-form update formulas while preserving approximate Gaussian conjugacy.}

Suppose at time $t-1$ we have $u \sim \mathcal{N}(m, S)$. At time $t$, we queried $z$ with corresponding question factor $v$.\footnote{For notational brevity, we omit random indices $I_t$, superscripts, and subscripts. The derivations are left exactly unchanged.} Because under our factor model, $z \sim \text{Bern}\left( \sigma\left( u^\top v \right) \right)$, by Bayes' Rule we have the following, where ``$\sim$" refers to equality up to an additive constant that does not depend on $u$, and $\sigma(x) := \frac{1}{1 + \exp(-x)} = \frac{\exp(x)}{1 + \exp(x)}$:
\begin{align*}
    p(u \mid z) \propto p(z \mid u) p(u) \propto \sigma\left( u^\top v \right)^{z} \left( 1 - \sigma\left( u^\top v \right) \right)^{1 - z} \exp\left( -\frac{1}{2}\left( u - m \right)^\top S^{-1} \left( u - m \right) \right)\\
    \implies \log p(u \mid z) \sim z \log\left( \sigma\left( u^\top v \right) \right) + \left( 1 - z \right) \log\left( 1 - \sigma\left( u^\top v \right) \right) -\frac{1}{2} \left( u - m \right)^\top S^{-1} \left( u - m \right)\\
    = z u^\top v - z \log\left( 1 + \exp\left( u^\top v \right) \right) - (1-z) \log\left( 1 + \exp(u^\top v) \right) -\frac{1}{2} \left( u - m \right)^\top S^{-1} \left( u - m \right)\\
    = z u^\top v - \log\left( 1 + \exp(u^\top v) \right) -\frac{1}{2}\left( u - m \right)^\top S^{-1} \left( u - m \right).
\end{align*}
To apply Laplace's approximation, we will require a second-order Taylor expansion of $\log p(u \mid z)$ about $u = m$. For now, we proceed by computing the gradient and Hessian of the exact log-posterior:
\begin{align*}
    \nabla_u \log p(u \mid z) = zv - \frac{\exp(u^\top v)}{1 + \exp(u^\top v)} v - S^{-1}(u - m) = zv - \sigma\left( u^\top v \right) v - S^{-1}(u - m),
\end{align*}
and, noting that $\nabla_x \sigma(x) = \sigma(x) (1 - \sigma(x))$, we have
\begin{align*}
    \nabla^2_u \log p(u \mid z) = -S^{-1} - \sigma\left( u^\top v\right)\left( 1 - \sigma\left( u^\top v\right) \right) vv^\top.
\end{align*}
Assembling our second-order Taylor expansion and defining $\phat := \sigma\left( u^\top v \right)$, let 
\begin{align*}
    b = \nabla_u \log p(u \mid z) \mid_{u = m} = zv - \sigma\left( u^\top v \right) v := \left( z - \phat \right) v,
\end{align*}
and 
\begin{align*}
    A = \nabla^2_u \log p(u \mid z) \mid_{u = m} = -S^{-1} - \phat \left( 1 - \phat \right) vv^\top.
\end{align*} 
Then, we have
\begin{align*}
    \log p(u \mid z) \approx \log p(m \mid z) + b^\top \left( u - m \right) + \frac{1}{2} \left( u - m \right)^\top A \left( u - m \right)\\
    \sim b^\top u + \frac{1}{2} u^\top A u - u^\top A m = \frac{1}{2} \left( u^\top A u - 2\left( A m - b \right)^\top u \right).
\end{align*}
After completing the square, we obtain the following 
\begin{align*}
    \log p(u \mid z) \sim \frac{1}{2} \left( u^\top A u - 2\left( A m - b \right)^\top u \right)\\ \sim -\frac{1}{2}\left( \left( u - A^{-1} \left( Am - b \right) \right)^\top \left( -A \right) \left( u - A^{-1} \left( Am - b \right) \right) \right)\\
    = -\frac{1}{2}\left( \left( u - \left( m - A^{-1} b \right) \right)^\top \left( -A \right) \left( u - \left( m - A^{-1} b \right) \right)  \right).
\end{align*}
Pattern-matching to the multivariate normal distribution, we deduce that
\begin{align*}
    u \mid z \overset{\cdot}{\sim} \mathcal{N}\left( m - A^{-1} b, -A^{-1} \right)\\ \eqD \mathcal{N}\left( m + \left( S^{-1} + \phat\left( 1 - \phat \right) vv^\top \right)^{-1} \left( z - \phat \right) v, \left( S^{-1} + \phat\left( 1 - \phat \right) vv^\top \right)^{-1} \right).
\end{align*}
For compute-efficient updates, the Sherman-Morrison formula tells us that
\begin{align*}
    \left( S^{-1} + \phat\left( 1 - \phat \right) vv^\top \right)^{-1} = S - \frac{\phat (1 - \phat) S vv^\top S}{1 + \phat (1 - \phat) v^\top S v}.
\end{align*}
Putting everything together,
\begin{align*}
    u \mid z \overset{\cdot}{\sim} \mathcal{N}\left( m + \left( S - \frac{\phat (1 - \phat) S vv^\top S}{1 + \phat (1 - \phat) v^\top S v} \right) \left( z - \phat \right) v,  S - \frac{\phat (1 - \phat) S vv^\top S}{1 + \phat (1 - \phat) v^\top S v} \right).
\end{align*}

\newpage

\subsection{Proof of Theorem \ref{thm:martingale-clt}}
\label{appendix:proof-martingale-clt}

To formalize asymptotics, we work under a triangular-array setup indexed by $m$. Specifically, we consider a sequence of problems $\{ \mathcal{P}^{(m)} \}_{m \geq 1}$, in which the bank size $N_q = N_q^{(m)}$, the budget $n_b = n_b^{(m)}$, and all involved quantities--$\{ z_{1}, \dots, z_{N_q} \}$, $\theta$, $\thetahat_{n_b}$, $\{ \phat^{(t)} \}_{0 \leq t < n_b}$, and $\{ q_t(\cdot) \}_{1 \leq t \leq n_b}$---may depend on $m$. All asymptotic statements are in the limit as $m \rightarrow \infty$ with $n_b^{(m)} \uparrow \infty$. Henceforth, we suppress the superscript $(m)$ for readability.

\subsubsection{Definitions}
Recall that our estimand is $\theta = \frac{1}{N_q} \sum_{j=1}^{N_q} z_j$ and $\thetahat_{n_b} = \frac{1}{n_b} \sum_{t=1}^{n_b} \phi_t$, where
\begin{align*}
    \phi_t = \frac{1}{N_q} \sum_{j=1}^{N_q} \phat_j^{(t-1)} + \frac{1}{N_q} \frac{z_{I_t} - \phat_{I_t}^{(t-1)}}{q_t(I_t)}.
\end{align*}
Let us define $\Delta_t := \phi_t - \frac{1}{N_q} \sum_{j=1}^{N_q} z_j = \phi_t - \theta$ and define the filtration $\F_t = \sigma(I_1, z_{I_1}, \dots, I_t, z_{I_t})$.

\subsubsection{Technical Assumptions/Conditions}
We enumerate the technical assumptions/conditions required to prove the asymptotic normality stated in Theorem~\ref{thm:martingale-clt}.

\begin{assumption}[Variance stabilization]\label{ass:var-stabilize}
The average conditional variance converges in probability to a positive deterministic limit $\sigma^2$ (which does not depend on $(m)$):
\begin{align*}
    \frac{1}{{n_b}} \sum_{t=1}^{n_b} \Var\left( \Delta_t \mid \F_{t-1} \right) \toP \sigma^2.
\end{align*}
\end{assumption}

\begin{assumption}[Lindeberg condition]\label{ass:lindeberg}
    The triangular array of problem instances in combination with sampling policies obey the Lindeberg condition. Specifically, for every $\epsilon > 0$,
    \begin{align*}
    \frac{1}{{n_b}} \sum_{t=1}^{n_b} \mathbb{E}[\Delta_t^2 \bdone(|\Delta_t| > \epsilon \sqrt{{n_b}}) \mid \F_{t-1}] \toP 0.
    \end{align*}
\end{assumption}

\begin{assumption}[Conditional variance control]\label{ass:pred-prob}
    The average conditional variance of the inverse-propensity-weighted predicted probabilities satisfies \[\frac{1}{n_b}\sum_{t=1}^{n_b}\Var\left(\frac{1}{N_q}\cdot\frac{\hat{p}^{(t-1)}_{I_t}}{q_t(I_t)} \mid \F_{t-1}\right) = o_p(n_b).\]
\end{assumption}
Assumptions~\ref{ass:var-stabilize}--\ref{ass:lindeberg} are the standard conditions under which a CLT for martingale difference arrays holds (c.f.~\citet[Corollary 3.2]{hall2014martingale}). Assumption~\ref{ass:pred-prob} is mild and is used to facilitate consistency of the variance estimator. Intuitively, in our setting of LLM evaluation with our proposed hybrid sampling policy (c.f. Eq.~\ref{eq:final-hybrid-sampling-policy}), the variance stabilization assumption is reasonable because as more outcome labels are observed, the model-factor estimate $(\uhat^{(t)}, \Sigmahat^{(t)})$ stabilizes, which stabilizes $\phat^{(t)}$, and hence $q_t(\cdot)$, too (because $(\alpha_t, \beta_t)$ also stabilize). In addition, the Lindeberg condition is satisfied, by construction, under our hybrid sampling policy because $\tau$-mixing lower-bounds $q_t(j)$, avoiding extreme inverse-probability weights, and since $z_j, \phat^{(t)}_j \in [0, 1]$, the increments $\Delta_t$ are uniformly bounded (by a constant depending only on $\tau$). The $\tau$-mixing also ensures the satisfaction of Assumption~\ref{ass:pred-prob}.

\subsubsection{Asymptotic Normality via Martingale CLT}
By substitution, we begin with the following:
\begin{equation*}
        \sqrt{{n_b}}\left( \thetahat_{n_b} - \theta \right) = \sqrt{{n_b}}\left( \frac{1}{{n_b}}\sum_{t=1}^{n_b} \phi_t - \frac{1}{N_q} \sum_{j=1}^{N_q} z_j \right).
\end{equation*}
Simplifying, we have
\begin{align*}
    \sqrt{{n_b}}\left( \thetahat_{n_b} - \theta \right) = \frac{1}{\sqrt{{n_b}}} \sum_{t=1}^{n_b} \phi_t - \frac{\sqrt{{n_b}}}{N_q} \sum_{j=1}^{N_q} z_j = \frac{1}{\sqrt{{n_b}}} \sum_{t=1}^{n_b} \phi_t - \frac{1}{\sqrt{{n_b}} N_q} \sum_{t=1}^{n_b} \sum_{j=1}^{N_q} z_j\\ = \sum_{t=1}^{n_b} \frac{1}{\sqrt{{n_b}}} \left( \phi_t - \frac{1}{N_q}\sum_{j=1}^{N_q} z_j \right) = \sum_{t=1}^{n_b} \frac{1}{\sqrt{n_b}} \Delta_t.
\end{align*}
Next, we will show that the $\Delta_t$ are martingale differences---i.e., that their cumulative sums form a martingale.

\begin{proposition}[Martingale differences and unbiasedness]
\label{prop:martingale} Define $M_t$ as follows:
\begin{equation}
M_t := \sum_{s=1}^t \left( \phi_s - \frac{1}{N_q}\sum_{j=1}^{N_q} z_j \right) := \sum_{s=1}^t \Delta_s.
\end{equation}
Then, $(M_t, \F_t)$ is a martingale. In particular, $\thetahat_{n_b}$ is unbiased for $\theta$ for every finite $n_b$.
\begin{proof}
    Substituting, we have
    \begin{align*}
        \mathbb{E}[\Delta_t \mid \F_{t-1}] &=  \mathbb{E}\left[ \phi_t - \frac{1}{N_q} \sum_{j=1}^{N_q} z_j \mid \F_{t-1} \right]\\ &= \frac{1}{N_q}  \mathbb{E}\left[ \sum_{j=1}^{N_q} \phat_j^{(t-1)} + \frac{z_{I_t} - \phat_{I_t}^{(t-1)}}{q_t(I_t)} - \sum_{j=1}^{N_q} z_j \mid \F_{t-1} \right]\\
        &= \frac{1}{N_q} \mathbb{E}\left[ \sum_{j=1}^{N_q} \left( \phat_j^{(t-1)} - z_j \right) + \frac{z_{I_t} - \phat_{I_t}^{(t-1)}}{q_t(I_t)} \mid \F_{t-1} \right]\\ &= \frac{1}{N_q}  \sum_{j=1}^{N_q} \left( \phat_j^{(t-1)} - z_j \right) + \frac{1}{N_q} \mathbb{E}\left[ \frac{z_{I_t} - \phat_{I_t}^{(t-1)}}{q_t(I_t)} \mid \F_{t-1} \right]\\ &= \frac{1}{N_q} \sum_{j=1}^{N_q} \left( \phat_j^{(t-1)} - z_j \right) + \frac{1}{N_q} \sum_{j=1}^{N_q} \frac{z_j - \phat_j^{(t-1)}}{q_t(j)} P(I_t = j \mid \F_{t-1})\\ &= \frac{1}{N_q} \sum_{j=1}^{N_q} \left( \phat_j^{(t-1)} - z_j \right) + \frac{1}{N_q} \sum_{j=1}^{N_q} \frac{z_j - \phat_j^{(t-1)}}{q_t(j)} q_t(j) = 0,
    \end{align*}
    thus confirming that we have a martingale. As a corollary, by summing and taking expectations, we find that $\mathbb{E}[\thetahat_{n_b}] = \theta$: i.e., that $\thetahat_{n_b}$ is unbiased for $\theta$ at any finite ${n_b}$.
\end{proof}
\end{proposition}

\newpage

Given that we have a martingale, we now proceed to obtain the asymptotic variance of $\thetahat_{n_b}$. To do, so, we need to first compute the conditional variances of our martingale increments. Because the $z_j$ are treated as fixed, by substitution we have
\begin{align*}
    \Var\left( \Delta_t \mid \F_{t-1} \right)\\ = \frac{1}{N_q^2} \Var\left( \sum_{j=1}^{N_q} \phat_j^{(t-1)} + \frac{z_{I_t} - \phat_{I_t}^{(t-1)}}{q_t(I_t)} - \sum_{j=1}^{N_q} z_j \mid \F_{t-1} \right) = \frac{1}{N_q^2} \Var\left( \frac{z_{I_t} - \phat_{I_t}^{(t-1)}}{q_t(I_t)} \mid \F_{t-1} \right)\\
    = \frac{1}{N_q^2} \mathbb{E}\left[ \left( \frac{z_{I_t} - \phat_{I_t}^{(t-1)}}{q_t(I_t)} \right)^2 \mid \F_{t-1} \right] - \frac{1}{N_q^2} \mathbb{E}\left[ \frac{z_{I_t} - \phat_{I_t}^{(t-1)}}{q_t(I_t)} \mid \F_{t-1} \right]^2 \\ = \frac{1}{N_q^2} \sum_{j=1}^{N_q} \left( \frac{z_j - \phat_j^{(t-1)}}{q_t(j)} \right)^2 P(I_t = j \mid \F_{t-1}) - \frac{1}{N_q^2} \left( \sum_{j=1}^{N_q} \frac{z_j - \phat_j^{(t-1)}}{q_t(j)} P(I_t = j) \right)^2\\ = \frac{1}{N_q^2} \left( \sum_{j=1}^{N_q} \frac{\left( z_j - \phat_j^{(t-1)}\right)^2}{q_t(j)} - \left( \sum_{j=1}^{N_q} \left( z_j - \phat_j^{(t-1)} \right) \right)^2 \right).
\end{align*}
Under Assumptions \ref{ass:var-stabilize}--\ref{ass:lindeberg}, the martingale central limit theorem (c.f., ~\citet[Corollary 3.2]{hall2014martingale}) implies
\begin{equation}
    \label{eq:final-martingale}
    \sqrt{{n_b}}\left( \thetahat_{n_b} - \theta \right) \toD \mathcal{N}\left( 0, \sigma^2 \right).
\end{equation}
Lemma~\ref{lemma:consistent-var} below shows that $\sigmahat_{n_b}^2 \overset{p}{\rightarrow} \sigma^2$, and thus the result follows by Slutsky's theorem.

\begin{lemma}\label{lemma:consistent-var}
    Define \begin{equation}
    \sigmahat_{n_b}^2 = \frac{1}{{n_b} N_q^2} \sum_{t=1}^{n_b} \frac{\left( z_{I_t} - \phat_{I_t}^{(t-1)}\right)^2}{q_t(I_t)^2} - \frac{1}{{n_b} N_q^2} \sum_{t=1}^{n_b} \left( \frac{1}{{n_b}}\sum_{s=1}^{n_b} \frac{z_{I_s}}{q_s(I_s)} - \sum_{j=1}^{N_q} \phat_j^{(t-1)} \right)^2.
\end{equation}

Then $\sigmahat_{n_b}^2 \overset{p}{\rightarrow} \sigma^2$.
\end{lemma}
\begin{proof}
    Define \[\hat{A}_{n_b} := \frac{1}{{n_b} N_q^2} \sum_{t=1}^{n_b} \frac{\left( z_{I_t} - \phat_{I_t}^{(t-1)}\right)^2}{q_t(I_t)^2}, \hat{B}_{n_b} :=\frac{1}{{n_b} N_q^2} \sum_{t=1}^{n_b} \left( \frac{1}{{n_b}}\sum_{s=1}^{n_b} \frac{z_{I_s}}{q_s(I_s)} - \sum_{j=1}^{N_q} \phat_j^{(t-1)} \right)^2\] as well as \[A := \frac{1}{n_b}\sum_{t=1}^{n_b}\frac{1}{N_q^2}  \sum_{j=1}^{N_q} \frac{\left( z_j - \phat_j^{(t-1)}\right)^2}{q_t(j)}, \; B := \frac{1}{n_b}\sum_{t=1}^{n_b}\frac{1}{N_q^2}\left( \sum_{j=1}^{N_q} \left( z_j - \phat_j^{(t-1)} \right) \right)^2\]

    To show that $\sigmahat_{n_b}^2 \overset{p}{\rightarrow} \sigma^2$, it suffices to show that (a) $\hat{A}_{n_b} = A + o_p(1)$ and (b) $\hat{B}_{n_b} = B+ o_p(1)$. We show these below:

    \begin{enumerate}[label=(\alph*)]
        \item Define $\tilde{\theta}^{(t-1)} := \frac{1}{N_q}\sum_{j=1}^{N_q}\hat{p}^{(t-1)}_j$. Under Assumptions~\ref{ass:var-stabilize}--\ref{ass:lindeberg}, \citet[][Theorem 2.23]{hall2014martingale} guarantees that \begin{align*}
        o_p(1) &= \label{eq:hall-heyde}
            \frac{1}{n_b}\sum_{t=1}^{n_b}\bigg[\left(\frac{1}{N_q} \frac{z_{I_t}-\hat{p}^{(t-1)}_{I_t}}{q_t(I_t)}+\tilde{\theta}^{(t-1)}-\theta\right)^2 \\ &-\left(\frac{1}{N_q^2}\sum_{j=1}^{N_q}\frac{\left(z_{j}-\hat{p}^{(t-1)}_j\right)^2}{q_t(j)} - (\tilde{\theta}^{(t-1)}-\theta)^2\right)\bigg]\\
            &= \hat{A}_{n_b} - A + \frac{1}{n_b}\sum_{t=1}^{n_b}\left[2(\tilde{\theta}^{(t-1)}-\theta)^2 - 2(\theta-\tilde{\theta}^{(t-1)})\left(\frac{1}{N_q} \frac{z_{I_t}-\hat{p}^{(t-1)}_{I_t}}{q_t(I_t)}\right)\right].
        \end{align*}

        To show that $\hat{A}_{n_b} - A = o_p(1)$ it suffices to show that \begin{equation}\label{ineqA}\frac{1}{n_b}\sum_{t=1}^{n_b}\left(\theta-\tilde{\theta}^{(t-1)}\right)\left(\theta-\tilde{\theta}^{(t-1)} -\frac{1}{N_q} \frac{z_{I_t}-\hat{p}^{(t-1)}_{I_t}}{q_t(I_t)}\right) = o_p(1).\end{equation}

        To show \eqref{ineqA}, we apply the Lenglart inequality (c.f., e.g., \citet[Lemma 3.7]{whitt2007proofs} or \citet[pp.~45]{jacod2012discretization}) to obtain 
        \begin{multline*}P\left(\left|\frac{1}{n_b}\sum_{t=1}^{n_b}\left(\theta-\tilde{\theta}^{(t-1)}\right)\left(\theta-\tilde{\theta}^{(t-1)} -\frac{1}{N_q} \frac{z_{I_t}-\hat{p}^{(t-1)}_{I_t}}{q_t(I_t)}\right)\right|  \geq \epsilon\right) \leq\\ \frac{\eta}{\epsilon} + P\left(\frac{1}{n_b^2}\sum_{t=1}^{n_b}\mathbb{E}\left[\left(\theta-\tilde{\theta}^{(t-1)}\right)^2\left(\theta-\tilde{\theta}^{(t-1)} -\frac{1}{N_q} \frac{z_{I_t}-\hat{p}^{(t-1)}_{I_t}}{q_t(I_t)}\right)^2\mid \F_{t-1}\right] \geq \eta\right)\end{multline*}
        for all $\epsilon,\eta > 0$. Hence, it suffices to show that $$\frac{1}{n_b^2}\sum_{t=1}^{n_b}\mathbb{E}\left[\left(\theta-\tilde{\theta}^{(t-1)}\right)^2\left(\theta-\tilde{\theta}^{(t-1)} -\frac{1}{N_q} \frac{z_{I_t}-\hat{p}^{(t-1)}_{I_t}}{q_t(I_t)}\right)^2\mid \F_{t-1}\right] = o_p(1),$$ which is true by virtue of the fact that $(\theta-\tilde{\theta}^{(t-1)})^2 \leq 1$ in combination with Assumption~\ref{ass:var-stabilize}.

        \item Expanding both squares, we find that \begin{align*}\hat{B}_{n_b} - B &= \left(\frac{1}{n_b}\sum_{s=1}^{n_b}\frac{1}{N_q}\cdot\frac{z_{I_s}}{q_s(I_s)}\right)^2-\theta^2 -2\tilde{\theta}^{(t-1)}\left[\frac{1}{n_b}\sum_{s=1}^{n_b}\frac{1}{N_q}\cdot\frac{z_{I_s}}{q_s(I_s)}-\theta\right]\\
        &= \left(\frac{1}{n_b}\sum_{s=1}^{n_b}\frac{1}{N_q}\cdot\frac{z_{I_s}}{q_s(I_s)}-\theta\right)^2 + 2\theta\left(\frac{1}{n_b}\sum_{s=1}^{n_b}\frac{1}{N_q}\cdot\frac{z_{I_s}}{q_s(I_s)}-\theta\right)\\ &-2\tilde{\theta}^{(t-1)}\left[\frac{1}{n_b}\sum_{s=1}^{n_b}\frac{1}{N_q}\cdot\frac{z_{I_s}}{q_s(I_s)}-\theta\right].\end{align*} 
        As $\theta, \tilde{\theta}^{(t-1)} \in [0,1]$, it suffices to show that $\frac{1}{n_b}\sum_{s=1}^{n_b}\frac{1}{N_q}\cdot\frac{z_{I_s}}{q_s(I_s)}-\theta\overset{p}{\rightarrow} 0.$ Since $\hat{\theta}_{n_b}-\theta\overset{p}{\rightarrow}0$ it further suffices to show that $\frac{1}{n_b}\sum_{t=1}^{n_b}\frac{1}{N_q}\frac{\hat{p}^{(t-1)}_{I_t}}{q_t(I_t)}-\frac{1}{N_q}\sum_{j=1}^{N_q}\hat{p}^{(t-1)}_j = o_p(1)$. Following the same argument as in the previous item, this is implied by \[\frac{1}{n_b^2}\sum_{t=1}^{n_b}\Var\left(\frac{1}{N_q}\cdot\frac{\hat{p}^{(t-1)}_{I_t}}{q_t(I_t)} \mid \F_{t-1}\right) \overset{p}{\rightarrow} 0,\] which is precisely Assumption~\ref{ass:pred-prob}.

    \end{enumerate}
\end{proof}

\subsection{Proof of Theorem \ref{thm:oracle-min-var}}
\label{appendix:proof-oracle-min-var}

Replacing $\phat_j^{(t-1)}$ with our non-time-varying oracle $p_j$'s, our estimator after $n_b$ sampling rounds becomes
\begin{align*}
    \thetahat_{{n_b},\text{oracle}} = \frac{1}{{n_b}} \sum_{t=1}^{n_b} \left( \theta^* + \frac{1}{N_q}\frac{z_{I_t} - p_{I_t}}{q_t(I_t)} \right) = \theta^* + \frac{1}{{n_b} N_q} \sum_{t=1}^{n_b} \frac{z_{I_t} - p_{I_t}}{q_t(I_t)}.
\end{align*}
Because $\theta^*$ is deterministic, to minimize the variance of $\thetahat_{{n_b},\text{oracle}}$, it suffices to minimize the variance of $M_{n_b} = \sum_{t=1}^{n_b} \frac{z_{I_t} - p_{I_t}}{q_t(I_t)}$. Define the filtration $\F_t = \sigma(I_1, z_{I_1}, \dots, I_t, z_{I_t})$. Writing out the variance, we have the following, noting that because $\thetahat_{{n_b},\text{oracle}}$ is unbiased for $\theta$, it must follow that $\mathbb{E}[M_{n_b}] = 0$ and thus
\begin{align*}
    \Var(M_{n_b}) = \mathbb{E}\left[ M_{n_b}^2 \right] - \mathbb{E}\left[ M_{n_b} \right]^2 = \mathbb{E}[M_{n_b}^2] = \mathbb{E}\left[ \left( \sum_{s=1}^{n_b} \frac{z_{I_s} - p_{I_s}}{q_s(I_s)} \right) \left( \sum_{t=1}^{n_b} \frac{z_{I_t} - p_{I_t}}{q_t(I_t)} \right) \right]\\
    = \mathbb{E}\left[ \sum_{s=1}^{n_b} \left( \frac{z_{I_s} - p_{I_s}}{q_s(I_s)} \right)^2 + 2 \sum_{1 \leq s < t \leq {n_b}} \left( \frac{z_{I_s} - p_{I_s}}{q_s(I_s)} \right) \left( \frac{z_{I_t} - p_{I_t}}{q_t(I_t)} \right) \right].
\end{align*}
Now, for any $s < t$, we deduce the following, because $\frac{z_{I_s} - p_{I_s}}{q_s(I_s)}$ and $q_t(\cdot)$ are $\F_{t-1}$-measurable:
\begin{align*}
    &\mathbb{E}\left[ \left( \frac{z_{I_s} - p_{I_s}}{q_s(I_s)} \right) \left( \frac{z_{I_t} - p_{I_t}}{q_t(I_t)} \right) \right]\\ &= \mathbb{E}\left[ \mathbb{E}\left[ \left( \frac{z_{I_s} - p_{I_s}}{q_s(I_s)} \right) \left( \frac{z_{I_t} - p_{I_t}}{q_t(I_t)} \right) \mid \F_{t-1} \right] \right] \\
    &= \mathbb{E}\left[ \left( \frac{z_{I_s} - p_{I_s}}{q_s(I_s)} \right) \mathbb{E}\left[ \left( \frac{z_{I_t} - p_{I_t}}{q_t(I_t)} \right) \mid \F_{t-1} \right] \right]\\ 
    &= \mathbb{E}\left[ \left( \frac{z_{I_s} - p_{I_s}}{q_s(I_s)} \right) \sum_{j=1}^{N_q} \mathbb{E}\left[ \left( \frac{z_{I_t} - p_{I_t}}{q_t(I_t)}\right) \mid I_t = j, \F_{t-1} \right] P(I_t = j \mid \F_{t-1})\right]\\
    &= \mathbb{E}\left[ \left( \frac{z_{I_s} - p_{I_s}}{q_s(I_s)} \right) \sum_{j=1}^{N_q} \mathbb{E}\left[ \left( \frac{z_{j} - p_{j}}{q_t(j)}\right) \mid I_t = j, \F_{t-1} \right] q_t(j) \right]\\
    &= \mathbb{E}\left[ \left( \frac{z_{I_s} - p_{I_s}}{q_s(I_s)} \right) \sum_{j=1}^{N_q} \mathbb{E}\left[ z_j - p_j \mid I_t = j, \F_{t-1} \right] \right] = 0,
\end{align*}
because for all $j$, $\mathbb{E}\left[ z_j - p_j \mid I_t = j, \F_{t-1} \right] = \mathbb{E}\left[ z_j - p_j\right] = 0$ by our oracle setup assumptions. Thus, by linearity, conditional independences, and variance of the Bernoulli, we have the following, noting that $q_t(\cdot)$ is $\F_{t-1}$-measurable:
\begin{align*}
    \Var(M_{n_b}) &= \sum_{t=1}^{n_b} \mathbb{E}\left[ \left( \frac{z_{I_t} - p_{I_t}}{q_t(I_t)} \right)^2  \right] = \sum_{t=1}^{n_b} \mathbb{E}\left[ \mathbb{E}\left[ \left( \frac{z_{I_t} - p_{I_t}}{q_t(I_t)} \right)^2 \mid \F_{t-1}  \right] \right]\\
    &= \sum_{t=1}^{n_b} \mathbb{E}\left[ \sum_{j=1}^{N_q} \mathbb{E}\left[ \left( \frac{z_{I_t} - p_{I_t}}{q_t(I_t)} \right)^2 \mid I_t = j, \F_{t-1}  \right] P(I_t = j \mid \F_{t-1}) \right]\\
    &= \sum_{t=1}^{n_b} \mathbb{E}\left[ \sum_{j=1}^{N_q} \mathbb{E}\left[ \left( \frac{z_{j} - p_{j}}{q_s(j)} \right)^2 \mid I_t = j, \F_{t-1}  \right] q_t(j) \right]\\
    &= \sum_{t=1}^{n_b} \mathbb{E}\left[ \sum_{j=1}^{N_q}\frac{\mathbb{E}\left[ \left( z_j - p_j \right)^2 \mid I_t = j, \F_{t-1} \right]}{q_t(j)} \right]\\
    &= \sum_{t=1}^{n_b} \mathbb{E}\left[ \sum_{j=1}^{N_q} \frac{p_j (1-p_j)}{q_t(j)} \right].
\end{align*}
At this point, note that we are minimizing a sum over $t$ of quantities that are nonnegative, as the $p_j$ are nonnegative and the $q_t(\cdot)$ are strictly positive by definition. From our final expression, there is also no temporal dependency between $q_s(\cdot)$ and $q_t(\cdot)$ for $s < t$. Finally, once provided with our choice of candidate $q_t(\cdot)$, the quantities inside the expectations are now deterministic. Combined, these three points tell us that (a) minimizing $\Var(M_{n_b})$ entails minimizing the instantaneous variance contributions $\sum_{j=1}^{N_q} p_j (1- p_j) / q_t(j)$ at each time step $s$; and (b) this instantaneous variance contribution minimization is the same across all timesteps. Specifically, at each time step $t$, we want to solve the following minimization problem, which is strictly convex because $p_j (1-p_j)$ are all nonnegative, and the $q_t(j)$ are all strictly positive:
\begin{align*}
    \min_{q_t(\cdot)} \sum_{j=1}^{N_q} \frac{p_j(1 - p_j)}{q_t(j)} \quad \text{s.t.} \quad q_t(j) > 0 \; \forall j, \; \sum_{j=1}^{N_q} q_t(j) = 1.
\end{align*}
Treating $q_t(\cdot)$ as a vector and setting up the Lagrangian, we have
\begin{align*}
    L(q_t(\cdot), \lambda) = \sum_{j=1}^{N_q} \frac{p_j(1 - p_j)}{q_t(j)} + \lambda \left( \sum_{j=1}^{N_q} q_t(j) - 1 \right).
\end{align*}
Partial-differentiating to obtain first-order conditions, we have
\begin{align*}
    \frac{\partial L(q_t(\cdot), \lambda)}{\partial q_t(j)} = -\frac{p_j(1-p_j)}{q_t(j)^2} + \lambda = 0 \implies {q_t^*(j)}^2 = \frac{p_j(1-p_j)}{\lambda} \implies q_t^*(j) \propto \sqrt{p_j(1-p_j)},
\end{align*}
exactly as desired, and emphasizing that the oracle-optimal $q_t^*(j)$ are not time-varying. Because of the strict convexity of our minimization problem, such a solution is the unique minimizer.

\subsection{Derivation of (Approximate) Maximal Variance-Reducing Index}
\label{appendix:proof-theta-glm-variance-reduction}

Under the idealized assumptions of Section \ref{subsection:active-learning-sampling-policy}, suppose at time $t$ that
\begin{align*}
    u \mid \F_{t-1} \sim \mathcal{N}\left( \uhat^{(t-1)}, \Sigmahat^{(t-1)} \right),
\end{align*}
and recall that
\begin{align*}
    \theta^* := g(u) = \frac{1}{N_q} \sum_{j=1}^{N_q} \sigma\left( u^\top v_{j} \right).
\end{align*}
Under our Laplace approximation update derivation setup in Appendix \ref{appendix:log-reg-derivation}, we know that with $\phat_j^{(t-1)} := \sigma\left( \uhat^{(t-1)}{}^\top v_j \right)$,
\begin{multline}
    \label{eqn:appendix-post-update-mean-cov}
    u \mid \F_{t-1}, z_j\\ \overset{\cdot}{\sim} \mathcal{N}\bigg( \uhat^{(t-1)} + \left( \Sigmahat^{(t-1)} - \frac{\phat_j^{(t-1)} (1 - \phat_j^{(t-1)}) \Sigmahat^{(t-1)} v_j v_j^\top \Sigmahat^{(t-1)}}{1 + \phat_j^{(t-1)} (1 - \phat_j^{(t-1)}) v_j^\top \Sigmahat^{(t-1)} v_j} \right) \left( z_j - \phat_j^{(t-1)} \right) v_j,\\  \Sigmahat^{(t-1)} - \frac{\phat_j^{(t-1)} (1 - \phat_j^{(t-1)}) \Sigmahat^{(t-1)} v_j v_j^\top \Sigmahat^{(t-1)}}{1 + \phat_j^{(t-1)} (1 - \phat_j^{(t-1)}) v_j^\top \Sigmahat^{(t-1)} v_j} \bigg).
\end{multline}
Define $\hat{w}^{(t-1)}_j := \phat_j^{(t-1)} \left( 1 - \phat_j^{(t-1)} \right) = \sigma\left( \uhat^{(t-1)}{}^\top v_j \right) \left( 1 - \sigma\left( \uhat^{(t-1)}{}^\top v_j \right) \right)$. Then,
\begin{align*}
    \Var\left( u \mid \F_{t-1}, z_j \right) = \Sigmahat^{(t-1)} - \frac{\hat{w}^{(t-1)}_j \Sigmahat^{(t-1)} v_j v_j^\top \Sigmahat^{(t-1)}}{1 + \hat{w}^{(t-1)}_j v_j^\top \Sigmahat^{(t-1)} v_j}.
\end{align*}
Now, differentiating and recalling that $\sigma'(z) = \sigma(z) \left( 1 - \sigma(z) \right)$, we have
\begin{align}
    \label{eqn:appendix-gradient-glm}
    \nabla_u g(u) = \frac{1}{N_q} \sum_{j=1}^{N_q} \sigma\left( u^\top v_j \right) \left( 1 - \sigma\left( u^\top v_j \right) \right) v_j.
\end{align}
By the normal approximation associated with the multivariate Delta method, it follows that
\begin{align*}
    \Var\left( \theta^* \mid \F_{t-1}, z_j \right) \approx \left( \nabla_u g(u)\mid_{u = \uhat^{(t-1)}} \right)^\top \Var\left( u \mid \F_{t-1}, z_j \right) \left( \nabla_u g(u)\mid_{u = \uhat^{(t-1)}} \right).
\end{align*}
Plugging in $\uhat^{(t-1)}$ instead of using the mean parameter update in Equation \ref{eqn:appendix-post-update-mean-cov} because we do not have access to $z_j$ yet (as we have not queried our next point yet), by a similar argument as above we have
\begin{align*}
    \Var\left( \theta^* \mid \F_{t-1} \right) \approx \left( \nabla_u g(u)\mid_{u = \uhat^{(t-1)}} \right)^\top \Var\left( u \mid \F_{t-1} \right) \left( \nabla_u g(u)\mid_{u = \uhat^{(t-1)}} \right).
\end{align*}
Subtracting, we obtain the following approximate variance reduction after substituting $\uhat^{(t-1)}$ into Equation \ref{eqn:appendix-gradient-glm}:
\begin{align*}
    &\Var\left( \theta^* \mid \F_{t-1} \right) - \Var\left( \theta^* \mid \F_{t-1}, z_j \right)\\ &\approx \frac{1}{N_q^2} \left( \sum_{j'=1}^{N_q} \hat{w}^{(t-1)}_{j'} v_{j'} \right)^\top \left( \frac{\hat{w}^{(t-1)}_j \Sigmahat^{(t-1)} v_j v_j^\top \Sigmahat^{(t-1)}}{1 + \hat{w}^{(t-1)}_j v_j^\top \Sigmahat^{(t-1)} v_j} \right)\left( \sum_{j'=1}^{N_q} \hat{w}^{(t-1)}_{j'} v_{j'} \right)\\
    &= \frac{1}{N_q^2} \frac{\hat{w}^{(t-1)}_j \left( v_j^\top \Sigmahat^{(t-1)} \sum_{j'=1}^{N_q} \hat{w}^{(t-1)}_{j'} v_{j'} \right)^2}{1 + \hat{w}^{(t-1)}_j v_j^\top \Sigmahat^{(t-1)} v_j}.
\end{align*}
Thus, the index $I_t$ that upon being queried most reduces the (approximate) posterior variance of $\theta^*$ is
\begin{align*}
    I_t = \argmax_j \frac{\hat{w}^{(t-1)}_j \left( v_j^\top \Sigmahat^{(t-1)} \frac{1}{N_q} \sum_{j'=1}^{N_q} \hat{w}^{(t-1)}_{j'} v_{j'} \right)^2}{1 + \hat{w}^{(t-1)}_j v_j^\top \Sigmahat^{(t-1)} v_j}.
\end{align*}

\newpage

\section{Investigating Sampling With/Without Replacement}
\label{appendix:without-replacement}

A natural question is whether we can sample questions \textit{without replacement}, since re-querying the same question appears wasteful. An instinctive ad-hoc implementation is to ``zero out" the sampling probabilities of previously-queried questions and renormalize before drawing the next index. However, making this change while keeping the same PAI estimator and variance estimate is not theoretically supported and can be empirically harmful at moderate-to-large budgets. Recall that Theorem \ref{thm:martingale-clt} was proven via a martingale CLT and that our proof treated $\Delta_t := \phi_t - \theta$, with
\begin{align*}
    \phi_t = \frac{1}{N_q} \sum_{j=1}^{N_q} \phat_j^{(t-1)} + \frac{1}{N_q} \frac{z_{I_t} - \phat_{I_t}^{(t-1)}}{q_t(I_t)},
\end{align*}
as martingale differences/increments with respect to the filtration $\F_t$. With with-replacement sampling, at each round we draw our next question $I_t \sim q_t(\cdot)$ (c.f. our hybrid sampling policy in Eq.~\ref{eq:final-hybrid-sampling-policy}) over the full question bank, and the inverse-probability correction in $\phi_t$ uses $q_t(I_t)$. This alignment is what enables our martingale structure and makes Assumptions \ref{ass:var-stabilize}--\ref{ass:lindeberg} plausible: as our model-factor estimate $(\uhat^{(t)}, \Sigmahat^{(t)})$ stabilizes as we observe more question labels, the predicted probabilities $\phat^{(t)}$ and, by extension, the sampling policy $q_t(\cdot)$, also stabilize. This is important because now it is reasonable that the average conditional variance $n_b^{-1} \sum_{t=1}^{n_b} \Var(\Delta_t \mid \F_{t-1})$ converges to a finite limit $v^2$---the crux of Assumption \ref{ass:var-stabilize}.
\begin{figure}[!h]
  \begin{center}
    \includegraphics[width=0.95\textwidth]{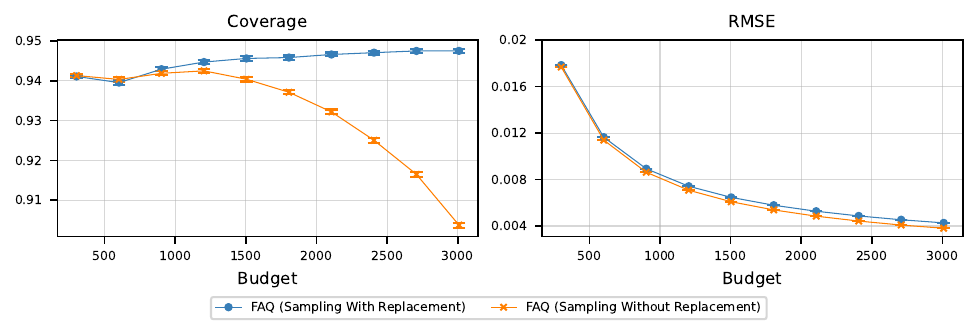}
    \caption{\textbf{Coverage and RMSE of FAQ with/without replacement on MMLU-Pro with fully-observed historical data.} \textbf{(Left)} coverage and \textbf{(Right)} RMSE of $\thetahat_{n_b}$ as a function of budget $n_b$, for FAQ sampling with/without replacement. Standard errors are shown, but negligible on the plot scale. Targeted nominal coverage for all experiments shown was $95\%$.}
    \label{appendix-version:with/without-replacement-coverage+rmse}
  \end{center}
\end{figure}

Under the ad-hoc without-replacement modification, we are no longer drawing the next question from $q_t$: rather, we are drawing from the restriction of $q_t$ to the remaining unseen questions $R_{t-1}$---let us call this restriction $\qtilde_t$. The key point is that $\qtilde_t(\cdot)$ is no longer aligned with PAI's estimator and variance estimate formulas---which both still use $q_t(I_t)$. As such, we no longer have martingale differences,\footnote{A consequence of this is also that $\thetahat_{n_b}$ is no longer unbiased for $\theta$.} thus hindering our martingale CLT argument and undermining our regularity conditions in Assumptions \ref{ass:var-stabilize}--\ref{ass:lindeberg}. For one, with our ad-hoc without-replacement modification, our $q_t(\cdot)$ are no longer lower-bounded, leading to unbounded inverse-probability corrections $1 / q_t(\cdot)$ and challenging Assumption \ref{ass:lindeberg}. Without access to the martingale CLT, our resulting normal approximation and variance estimate can become miscalibrated.

Empirically, Figure \ref{appendix-version:with/without-replacement-coverage+rmse} (Left) corroborates this: at moderate-to-large budgets, the ad-hoc without-replacement variant suffers from substantial miscoverage at moderate to large budgets. At low budgets, the difference in coverage is smaller because repeat sampling of questions is rarer, and thus the two sampling schemes behave similarly. We also observe from Figure \ref{appendix-version:with/without-replacement-coverage+rmse} (Right) that the ad-hoc without-replacement modification yields little improvement in RMSE (point estimate performance).

For these reasons, we recommend sampling with replacement when using FAQ/PAI to preserve theoretically-guaranteed asymptotic uncertainty quantification, with minimal loss in point-estimation performance.

\newpage

\section{Additional Experimental Details}
\label{appendix:experimental-details}

\textbf{Datasets}\;\;\; We compile our historical data from the HuggingFace Open LLM Leaderboard \citep{open-llm-leaderboard-v2, myrzakhan2024open}, which is licensed under \href{https://huggingface.co/spaces/open-llm-leaderboard/open_llm_leaderboard/blob/main/README.md}{Apache-2.0}. We create one benchmark from MMLU-Pro \citep{wang2024mmlu} itself with $N_q = 12032$ questions, and aggregate the remaining datasets---BBH \citep{suzgun2023challenging}, GPQA \citep{rein2024gpqa}, IFEval \citep{zhou2023instruction}, MATH \citep{hendrycksmath2021}, and MuSR \citep{spraguemusr}---into a second benchmark with $N_q = 9574$ questions.

\textbf{Historical splits and missingness}\;\;\; For non-cold-start experiments comparing FAQ to the strongest baselines and uniform sampling, we consider the fully-observed historical data setting and the following $(n_{\text{full-obs}}, p_{\text{obs}})$ pairs: $(50, 0.1)$, $(200, 0.1)$, $(800, 0.1)$, and a particularly extremely-sparse $(0, 0.001)$ setting. 

For our traditional active inference ablation experiments, we consider the fully-observed data setting and the following $(n_{\text{full-obs}}, p_{\text{obs}})$ pairs: $(0, 0.01)$, $(0, 0.001)$, $(0, 0.0001)$, and $(0, 0.00001)$. We chose this range of missingness settings after observing the results of FAQ against the strongest baselines and uniform sampling, and desired missingness-settings that would best distinguish between FAQ and the traditional active inference ablation.

\textbf{Factor model selection}\;\;\; For non-cold-start experiments, we use the AdamW optimizer with weight decay $\lambda$, learning rate $10^{-2}$, and maximum iterations $2000$, implemented in PyTorch \citep{paszke2019pytorch}, with all experiments run on either a single NVIDIA A100 (40 GB) or NVIDIA H100 (80 GB). We sweep over latent factor dimension $k \in \{ 8, 16, 32, 64 \}$ and weight decay $\lambda \in \{ 10^1, 10^0, 10^{-1}, 10^{-2}, 10^{-3}, 10^{-4}, 10^{-5} \}$.

\textbf{FAQ hyperparameters}\;\;\; For all experiments, we sweep over the maximum tempering exponent $\beta_0 \in \{ 0.25, 0.5, 0.75, 1.0 \}$, the tempering governor $\gamma \in \{ 0.0, 0.05, 0.25, 0.5,0.75 \}$, the exploration-to-exploitation governor $\rho \in \{ 0.0, 0.05, 0.25, 0.5, 0.75 \}$, and the uniform-mixing strength $\tau \in \{ 0.05, 0.25, 0.5, 0.75 \}$.

\textbf{Baselines and active inference ablation}\;\;\; We sweep over uniform mixing strength $\tau \in \{ 0.05, 0.25, 0.5, 0.75 \}$.

\textbf{Cold-start and cross-suite transfer}\;\;\; Because no pre-existing target-suite (i.e., GPQA (All)) historical matrix is available for factor-model selection, we fit the initial factor model only on fully-observed MMLU-Pro historical data, using a fixed a priori configuration with factor dimension $k=8$, weight decay $\lambda = 1.0$, and $2000$ AdamW iterations. We then freeze the MMLU-Pro-fitted historical model factors $u_i$ and use their empirical mean and covariance to initialize the prior for new models' factors on GPQA (All). To initialize GPQA (All) question factors, we compute embeddings of all MMLU-Pro and GPQA (All) question texts using OpenAI's \texttt{text-embedding-3-large} model. For each GPQA (All) question, we find its $8$ nearest MMLU-Pro questions using cosine similarity in embedding space, and set its question factor to the corresponding similarity-weighted average of their MMLU-Pro question factors. FAQ sampling hyperparameters are tuned analogously to the procedures described above and in the main text, using a held-out set of $443$ validation models that we assume are available to query, and minimizing CI width only. These validation models are excluded from final evaluation: coverage and efficiency are evaluated only on the $2.2$K test models.

All source code and analysis routines for reproducing all results in this paper can be found at \url{https://github.com/skbwu/efficiently-evaluating-llms}.

\newpage

\section{Traditional Active Inference Ablation}
\label{appendix:traditional-active-inference-ablation}

\begin{figure}[!h]
  \begin{center}
    \includegraphics[width=0.95\textwidth]{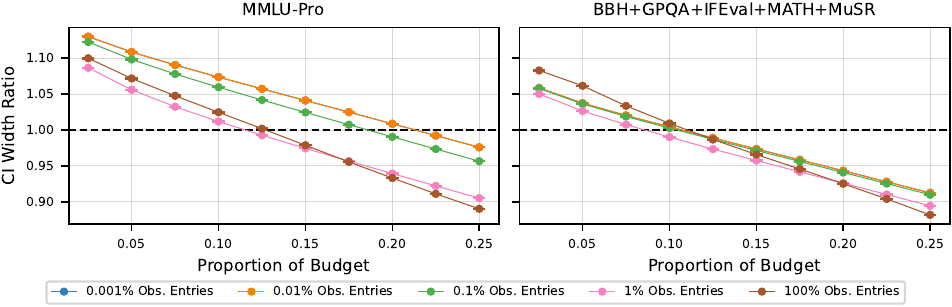}
    \caption{\textbf{CI-width ratio (post-hoc best traditional active inference ablation / FAQ) across budgets/missingness} Missingness uses $n_{\text{full-obs}}=0$ and MCAR $p_{\text{obs}}$ (legend values) on the remaining entries. Standard errors are shown, but negligible at the plot scale. \textbf{Width ratios $> 1$ favor FAQ (narrower CIs).}}
    \label{fig4_active-inference-ablation_widths}
  \end{center}
\end{figure}

To isolate the gains of PAI and our hybrid sampling (beyond the factor model), we compare FAQ to a ``traditional" active inference ablation that follows \citet{zrnic2024active}'s sequential (``nature's order")\footnote{Here, ``nature's order" refers to the benchmarks' native question orderings in HuggingFace's Open LLM Leaderboard.} procedure while using the \textit{same} factor-model predictions. We report CI-width ratios (rather than ESS) since they compare tightness directly across budgets without converting to a uniform-equivalent ``budget savings” scale. Figure \ref{fig4_active-inference-ablation_widths} reveals a clear phase shift. At low-budgets (e.g, $< 10\%$ of the bank), FAQ yields substantially narrower CIs; at larger budgets, the traditional ablation eventually overtakes FAQ by similar margins. This is not discouraging: large-scale evaluation ($1$K+ models, $10$K+ questions) typically emphasizes the low-budget regime where FAQ is strongest. It also suggests a natural extension---a three-stage hybrid policy: active-learning early, oracle-style variance reduction mid-budget, and traditional active inference at the end---combining the best of both worlds.

Three mechanisms likely drive the shift. First, FAQ can proactively target informative questions early on via its active-learning component, whereas traditional active inference can only label/skip items in a fixed stream and its $\pi_t(t) \propto \min(\phat_t^{(t-1)}, 1-\phat_t^{(t-1)})$ rule contains no explicit active-learning objective---explaining the low-budget gap. Second, at large budgets, FAQ's sampling with replacement can resample questions, reducing efficiency relative to stream-based procedures that touch each question at most once. Third, our tuned $\tau$-mixing suggests variance control dominates at high budgets: the best post-hoc traditional ablation consistently uses $\tau = 0.75$ (near-uniform query/label probabilities). By contrast, FAQ's train/val tuning often selects much smaller $\tau$ (e.g., $0.05$ for many budgets $\geq 10\%$), yielding spikey/highly-concentrated $q_t$. After the factor  model has largely converged, this extra ``targeting" offers little learning benefit, but increases AIPW variance through larger weights $1 / q_t(I_t)$, helping explain why the traditional ablation can overtake FAQ at large budgets. Because FAQ targets cost savings, we focus on the low-budget regime.

\newpage

\section{Additional Results}
\label{appendix:additional-results}

\begin{figure}[!h]
  \begin{center}
    \includegraphics[width=0.95\textwidth]{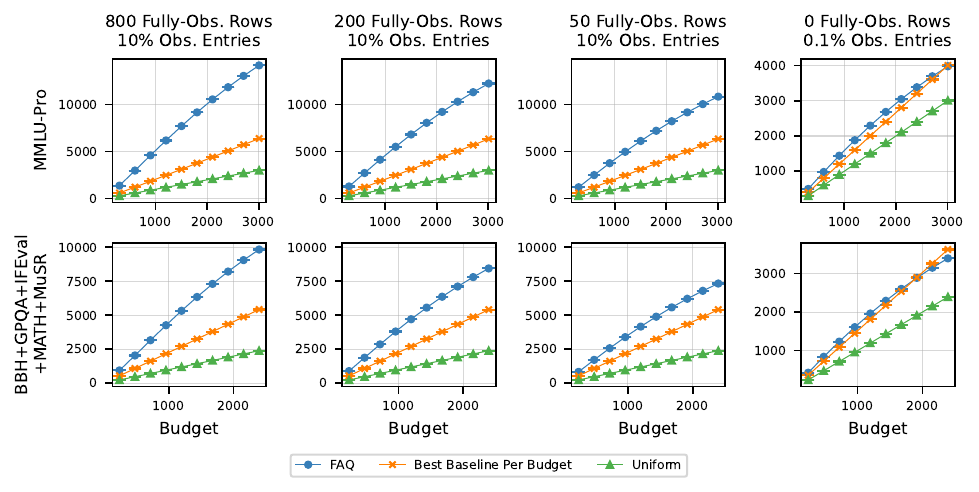}
    \caption{\textbf{ESS under missing historical data on MMLU-Pro and BBH+GPQA+IFEval +MATH+MuSR.} ESS vs. budget for FAQ, the strongest baseline per budget (selected post-hoc), and uniform sampling, per missingness setting. Missingness is parameterized by $n_{\text{full-obs}}$ fully-observed historical rows (out of $2.2$K) and MCAR entrywise observation probability $p_{\text{obs}}$ on the remaining rows (e.g., $n_{\text{full-obs}} = 50$ and $p_{\text{obs}} = 0.1$ $\sim$ $12\%$ observed). Standard errors are shown, but negligible on the plot scale. This figure shows the full results corresponding to Figure \ref{fig2_ess_sparsity-settings} in the main text.}   
    \normalsize
    \label{appendix-version:fig2_ess_sparsity-settings}
  \end{center}
\end{figure}

\begin{figure}[!b]
  \begin{center}
    \includegraphics[width=0.95\textwidth]{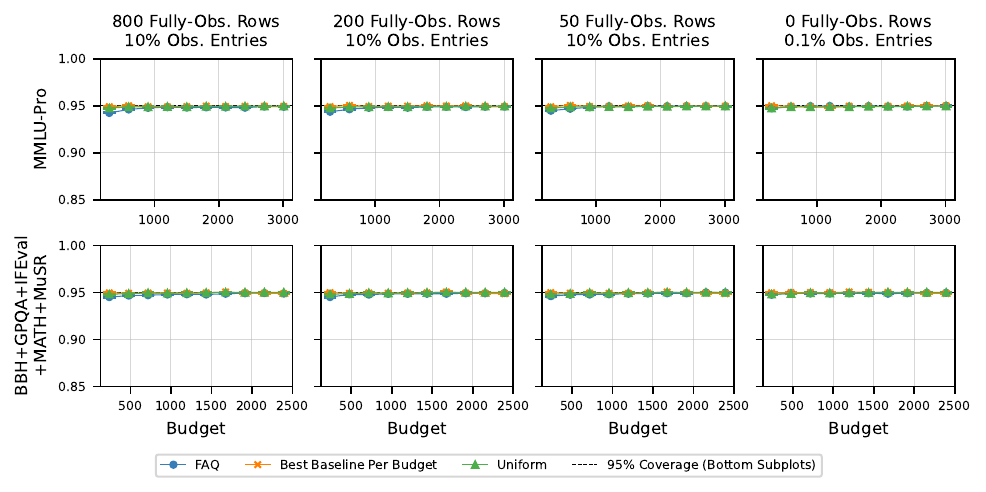}
    \caption{\textbf{Coverage under missing historical data on MMLU-Pro and BBH+GPQA+IFEval +MATH+MuSR.} Coverage vs. budget for FAQ, the strongest baseline per budget (selected post-hoc), and uniform sampling, per missingness setting. Missingness is parameterized by $n_{\text{full-obs}}$ fully-observed historical rows (out of $2.2$K) and MCAR entrywise observation probability $p_{\text{obs}}$ on the remaining rows (e.g., $n_{\text{full-obs}} = 50$ and $p_{\text{obs}} = 0.1$ $\sim$ $12\%$ observed). Standard errors are shown, but negligible on the plot scale. This figure confirms that the results in Figure \ref{fig2_ess_sparsity-settings} in the main text were achieved with valid coverage.}   
    \normalsize
    \label{appendix-version:full-coverages.}
  \end{center}
\end{figure}

\begin{figure}[!h]
  \begin{center}
    \includegraphics[width=0.95\textwidth]{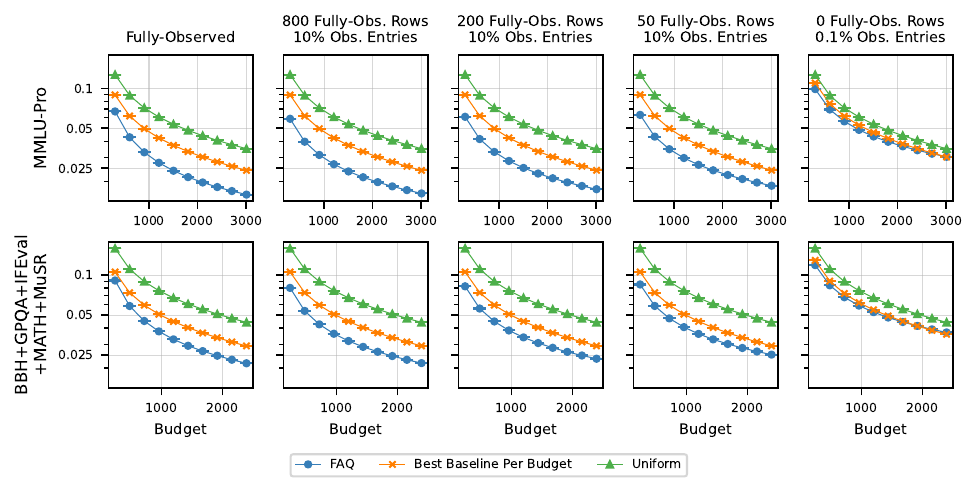}
    \caption{\textbf{Confidence interval (CI) widths under missing historical data on MMLU-Pro and BBH+GPQA+IFEval+MATH+MuSR.} CI width vs. budget for FAQ, the strongest baseline per budget (selected post-hoc), and uniform sampling, per missingness setting. Missingness is parameterized by $n_{\text{full-obs}}$ fully-observed historical rows (out of $2.2$K) and MCAR entrywise observation probability $p_{\text{obs}}$ on the remaining rows (e.g., $n_{\text{full-obs}} = 50$ and $p_{\text{obs}} = 0.1$ $\sim$ $12\%$ observed). Standard errors are shown, but negligible on the plot scale. \textbf{Smaller CI widths indicate stronger performance.}}    
    \label{appendix-version:full-widths}
    \normalsize
  \end{center}
\end{figure}

\begin{figure}[!h]
  \begin{center}
    \includegraphics[width=0.95\textwidth]{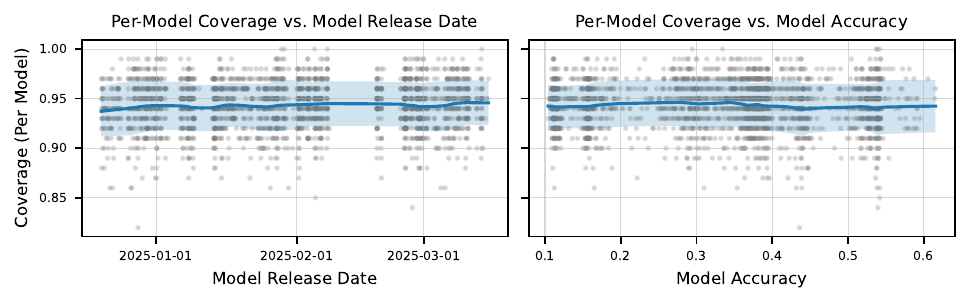}
    \caption{\textbf{Per-model coverages on MMLU-Pro with fully-observed historical data at budget $7.5\%$.} Per-model coverages (averaged over $100$ seeds) vs. \textbf{(Left)} model release date and \textbf{(Right)} true model accuracy. Gray dots: individual models; solid-blue curve: locally-weighted mean with a $K=501$ nearest-neighbor Gaussian kernel; light-blue band: local $\pm 1$ SD band. This figure shows the MMLU-Pro counterpart to Figure \ref{fig3_per-model-coverages-7.5-bbh} in the main text.}
    \label{appendix-version:fig3_per-model-coverages-7.5-mmlu}
  \end{center}
\end{figure}

\begin{figure}[!h]
  \begin{center}
  \includegraphics[width=0.95\textwidth]{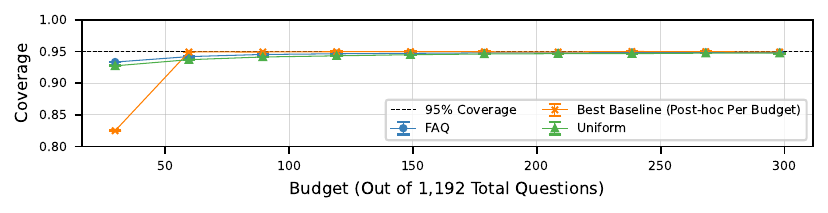}
    \caption{\textbf{Coverage for cold-start FAQ on GPQA (All).} Coverage vs. budget for cold-start FAQ \textit{with zero target-suite historical data  for factor-model fitting}, the strongest post-hoc baseline per budget \textit{with access to $0.1\%$ target-suite historical data}, and uniform sampling. Standard errors are shown, but negligible on the plot scale. This figure confirms that main text Figure \ref{figx:cold-start}'s results were achieved with valid coverage.}
    \label{figx:cold-start_appendix}
  \end{center}
\end{figure}

\FloatBarrier

\newpage

\section{Additional Exploratory Robustness Experiments}
\label{appendix-exploratory-experiments}

In this section, we report additional exploratory experiments examining FAQ's robustness to structured (as opposed to MCAR) historical-data missingness and historical-label corruption, as well as its sensitivity to factor-model predictive accuracy. Given their exploratory nature, we focus on the BBH+GPQA+IFEval+MATH+MuSR suite ($N_q = 9574$) at three budgets, where applicable: $2.5\%$, $7.5\%$, and $12.5\%$. For the historical-data perturbation experiments, we retain the full FAQ pipeline with proper train/validation/test splits, using the same perturbed historical data $H$ for all baselines, and $100$ randomized seeds per setting. We organize our specific experiments by research question.

\paragraph{How does FAQ perform under historical-data missingness by model family?} We removed all fine-tuned Qwen2.5 and QwQ (i.e., Qwen2.5's thinking/reasoning variant) variants from $H$. We chose the Qwen2.5/QwQ family because it was the largest base model family, compared to Llama-3.1 and Gemma-2. This constituted the removal of $131$ of the $2.2$K historical models. On the $321$ Qwen2.5/QwQ-family test models at budgets $2.5\%$, $7.5\%$, and $12.5\%$, FAQ achieved ESS multipliers over uniform sampling of $2.54$, $3.31$, and $3.57$, respectively, outperforming the best post-hoc selected baseline, whose ESS multipliers never exceeded $2.02$, while maintaining coverage within $0.6\%$ of the $95\%$ target. Thus, FAQ can pool information across families to remain effective when a major family is absent from $H$.

\paragraph{How does FAQ perform under historical missingness by release date?} We sorted $H$ by release date and removed its newest $50\%$ of models, an extreme perturbation given that our original historical/test split was already chronological. Nonetheless, across all $2.2$K test models at budgets $2.5\%$, $7.5\%$, and $12.5\%$, FAQ achieved ESS multipliers over uniform sampling of $3.32$, $3.59$, and $3.58$, respectively, outperforming the best post-hoc selected baseline, whose ESS multipliers never exceeded $2.23$, while maintaining coverage within $0.6\%$ of the $95\%$ target.

\paragraph{How does FAQ perform under random label-flipping corruption?} We uniformly flipped $5\%$ of labels in the historical data as a pessimistic stress test. Across all $2.2$K test models at budgets $2.5\%$, $7.5\%$, and $12.5\%$, FAQ achieved ESS multipliers over uniform sampling of $3.85$, $4.30$, and $4.36$, respectively, outperforming the best post-hoc selected baseline, whose ESS multipliers never exceeded $2.28$, while maintaining coverage within $0.5\%$ of the $95\%$ target. These results are slightly better than those in Figure \ref{fig1:ess+coverage_fully-observed}. One possible explanation is that random label-flipping regularizes the factor-model fit and/or tempers extreme probabilistic predictions, improving test-set probability calibration and hence sampling efficiency.

\paragraph{How does FAQ perform under benchmark-specific label corruption?} We concentrated the same $5\%$ overall corruption on BBH, the largest benchmark ($5761$ questions), corresponding to $8.30\%$ corruption within BBH. Across all $2.2$K test models at budgets $2.5\%$, $7.5\%$, and $12.5\%$, FAQ achieved ESS multipliers over uniform sampling of $3.75$, $4.17$, and $4.22$, respectively, outperforming the best post-hoc selected baseline, whose ESS multipliers never exceeded $2.27$, while maintaining coverage within $0.5\%$ of the $95\%$ target. These results are also slightly better than those in Figure \ref{fig1:ess+coverage_fully-observed}, plausibly for the same reason.

\paragraph{How does FAQ perform under model-family-specific label corruption?} We concentrated the $5\%$ overall corruption rate on the $131$ historical Qwen2.5/QwQ-family models, corresponding to $84.3\%$ corruption within this family. On the $321$ Qwen2.5/QwQ-family test models at budgets $2.5\%$, $7.5\%$, and $12.5\%$, FAQ achieved ESS multipliers over uniform sampling of $2.69$, $3.47$, and $3.71$, respectively, outperforming the best post-hoc selected baseline, whose ESS multipliers never exceeded $2.03$, while maintaining coverage within $0.6\%$ of the $95\%$ target.

\paragraph{How robust are FAQ's empirical efficiency gains under factor model misspecification?} Because the true response model is unknown, we quantify practical misspecification using pre-query binary prediction accuracy: the fraction of question outcomes correctly predicted before observing any new-model labels. On BBH+GPQA+IFEval+MATH+MuSR with fully observed $H$, we stratified the $2.2$K test models into quartiles by this accuracy. Predictions use the empirical normal prior mean computed from the historical model factors: lower accuracy therefore indicates weaker transfer to the new model. At a $7.5\%$ budget, FAQ achieved mean ESS multipliers over uniform of $2.64$ in the worst-predicted quartile and $5.43$ in the best-predicted quartile, averaged over models within each quartile. Their mean pre-query accuracies were $66.21\%$ and $80.20\%$, respectively. FAQ also outperformed the strongest post-hoc selected baseline in every quartile. Across quartiles, coverage ranged from $94.64\%$ to $94.84\%$, near the $95\%$ target. Thus, poorer factor predictions reduce attainable efficiency gains, but FAQ remains more efficient than the strongest baselines while preserving near-nominal coverage.

\end{document}